\documentclass[lettersize,journal]{IEEEtran}
\usepackage{amsmath,amsfonts}
\usepackage{algorithmic}
\usepackage[ruled]{algorithm2e}
\usepackage{array}
\usepackage[caption=false,font=normalsize,labelfont=sf,textfont=sf]{subfig}
\usepackage{textcomp}
\usepackage{stfloats}
\usepackage{url}
\usepackage{verbatim}
\usepackage{graphicx}
\usepackage{cite}
\usepackage{bm}
\usepackage{bbm}
\usepackage{amsthm}
\usepackage{amsmath}
\usepackage{amssymb}
\usepackage{enumitem}
\usepackage[small,bf]{caption}

\usepackage{hyperref}
\usepackage{pifont}
\usepackage[table]{xcolor}
\usepackage{multirow}
\usepackage{threeparttable}
\usepackage{orcidlink}
\usepackage{placeins}
\usepackage{dsfont}
\usepackage{subfig}
\usepackage{subcaption}
\usepackage{booktabs}
\usepackage{ragged2e}
\usepackage{multicol}
\usepackage{float}
\usepackage{stfloats}
\usepackage{placeins}

\newcommand{\JustifyAlgo}[1]{
  \parbox[t]{\dimexpr\linewidth-\algorithmicindent}{\justifying #1}
}

\SetKwProg{sub}{Subroutine}{}{}

\allowdisplaybreaks

\DeclareMathOperator*{\argmin}{arg\,min}

\newcommand{\cmark}{\textcolor{green}{\ding{52}}}
\newcommand{\xmark}{\textcolor{red}{\ding{55}}}

\newtheorem{assumption}{Assumption}
\newtheorem{lemma}{Lemma}
\newtheorem{theorem}{Theorem}
\newtheorem{definition}{Definition}
\newtheorem{proposition}{Proposition}
\newtheorem{remark}{Remark}

\begin{document}

\title{Greedy Low-Rank Gradient Compression for Distributed Learning
with Convergence Guarantees
}

\author{Chuyan  Chen$^{*}$\thanks{$^{*}$Equal contributions. $^\dagger$Corresponding authors.}\thanks{Chuyan Chen is with School of Mathematical Sciences, Peking University, Beijing 100871, China (e-mail: chuyanchen@stu.pku.edu.cn).}~\orcidlink{0009-0005-0273-6649}, Yutong He$^{*}$\thanks{Yutong He is with Center for Data Science, Peking University, Beijing 100871, China (e-mail: yutonghe@pku.edu.cn).}~\orcidlink{0009-0002-5078-6454}, Pengrui Li\thanks{Pengrui Li is with School of Software, Beihang University, Beijing 100191, China. (e-mail: 22377115@buaa.edu.cn).}~\orcidlink{0009-0000-3364-0002}, Weichen Jia\thanks{Weichen Jia is with School of Life Sciences, Peking University, Beijing 100871, China. (e-mail: 2100012106@stu.pku.edu.cn).}~\orcidlink{0009-0007-3025-7922}, and Kun Yuan$^{\dagger}$\thanks{Kun Yuan is with Center for Machine Learning Research, Peking University, Beijing 100871, China (e-mail: kunyuan@pku.edu.cn).}~\orcidlink{0000-0001-8394-8187}}

\markboth{}
{Shell \MakeLowercase{\textit{et al.}}: A Sample Article Using IEEEtran.cls for IEEE Journals}

\maketitle

\begin{abstract}
Distributed optimization is pivotal for large-scale signal processing and machine learning, yet communication overhead remains a major bottleneck. Low-rank gradient compression, in which the transmitted gradients are approximated by low-rank matrices to reduce communication, offers a promising remedy. Existing methods typically adopt either randomized or greedy compression strategies: randomized approaches project gradients onto randomly chosen subspaces, introducing high variance and degrading empirical performance; greedy methods select the most informative subspaces, achieving strong empirical results but lacking convergence guarantees. To address this gap, we propose GreedyLore—the first \underline{Greedy} \underline{L}ow-\underline{R}ank gradi\underline{e}nt compression algorithm for distributed learning with rigorous convergence guarantees. GreedyLore incorporates error feedback to correct the bias introduced by greedy compression and introduces a semi-lazy subspace update that ensures the compression operator remains contractive throughout all iterations. With these techniques, we prove that GreedyLore achieves a convergence rate of $\mathcal{O}(\sigma/\sqrt{NT} + 1/T)$ under standard optimizers such as MSGD and Adam—marking the first linear speedup convergence rate for low-rank gradient compression. Extensive experiments are conducted to validate our theoretical findings.
\end{abstract}


\section{Introduction}
\IEEEPARstart{D}{istributed} optimization is a promising paradigm for addressing large-scale problems in signal processing and machine learning. In distributed algorithms, each computing node processes its local data while collaborating with others to minimize a global loss function. This approach mitigates the computational burden on individual nodes, reduces memory requirements, and enables efficient parallel computation. In addition to its applications in edge computing~\cite{shi2016edge, satyanarayanan2017emergence, zhou2019edge}, federated learning~\cite{konevcny2016federated, konevcny2016federated2, mcmahan2017communication}, and robotic control~\cite{bullo2009distributed}, distributed optimization is particularly useful for training deep neural networks~\cite{li2014communication, goyal2017accurate, brown2020language, touvron2023llama}. By partitioning the computation across multiple nodes, distributed optimization allows for efficient training of massive deep neural models.

This paper studies the following distributed stochastic optimization problem involving a  matrix variable $\bm{X}\in \mathbb{R}^{m\times n}$:
\begin{align}
\min_{\bm{X}\in \mathbb{R}^{m\times n}}\ f(\bm{X}):=\frac{1}{N}\sum_{i=1}^N \left[f_i(\bm{X}) := \mathbb{E}_{{\bm{\xi}}\sim\mathcal{D}_i}F_i(\bm{X}, {\bm{\xi}}) \right].
    \label{eq:formulation}
\end{align}
Matrix variables arise naturally in many modern applications—for example, the model weights in each layer of a deep neural network are typically represented as matrices. Notably, when \(n = 1\), Problem~\eqref{eq:formulation} reduces to a standard stochastic optimization problem with a vector variable. In this formulation, the notation $N$ denotes the number of computing nodes, and ${\bm{\xi}}$ is a random variable representing the data drawn from the local distribution $\mathcal{D}_i$. Each node $i$ can compute the stochastic gradient $\nabla F_i(\bm{X}; {\bm{\xi}})$ of its loss function; however, communication with other nodes is required to obtain global gradients. Since the local data distributions $\{\mathcal{D}_i\}_{i=1}^N$ may differ across nodes, the local loss functions $f_i(\bm{X})$ are not necessarily identical, i.e., $f_i(\bm{X}) \neq f_j(\bm{X})$ in general.

\subsection{Low-Rank Communication Compression} 
Many distributed algorithms require individual nodes to transmit full-size gradients to a central server for model parameter updates~\cite{li2014communication, smola2010architecture, strom2015scalable, gibiansky2017bringing}. However, the high dimensionality of these gradients introduces substantial communication overhead in each iteration, significantly limiting the efficiency and scalability of distributed learning~\cite{seide20141, chilimbi2014project}. To mitigate this bottleneck, various communication compression techniques have been proposed~\cite{alistarh2017qsgd, bernstein2018signsgd, stich2018sparsified, richtarik2021ef21, huang2022lower}. These methods reduce the communication cost per iteration by transmitting compressed tensors instead of full ones. The two most prominent compression strategies are quantization and sparsification.
Quantization~\cite{alistarh2017qsgd, Horvath2019StochasticDL, seide20141} maps input tensors from a potentially large or continuous domain to a smaller, discrete set—examples include 1-bit quantization~\cite{seide20141} and natural compression~\cite{Horvath2019StochasticDL}. It is particularly effective in scenarios where the communicated entries are relatively evenly distributed, enabling efficient low-bit encoding across the value range.
In contrast, sparsification~\cite{wangni2018gradient, stich2018sparsified, safaryan2021smoothness} is more suitable when most entries are close to zero. It achieves compression by selectively discarding a large portion of the less informative entries, producing sparse representations. Popular examples include rand-$K$ and top-$K$  compressors~\cite{stich2018sparsified}.

This paper focuses on an alternative form of communication compression—low-rank compression—which leverages the inherent low-rank structures of the transmitted variables to reduce communication costs. Low-rank structures are ubiquitous in practice, with deep neural networks being a prominent example. For instance, ~\cite{hu2022lora} demonstrates that parameter updates during fine-tuning large language models (LLMs) tend to be low-rank, motivating the development of the well-known LoRA fine-tuning method. Similarly, recent work~\cite{he2024subspace,chen2024enhancing,zhao2025separate} observe that gradients in neural network pre-training also possess low-rank structures. Additionally, SL-Train~\cite{han2024sltrain} finds that model weights exhibit both low-rank and sparse patterns during pre-training, while DeepSeek-V3~\cite{liu2024deepseek} suggests that compressing KV-cache activations via low-rank approximation can substantially reduce memory usage. In such scenarios, low-rank compression is often more suitable than quantization or sparsification, as it more effectively captures the underlying structure of the transmitted variables.

\newcolumntype{P}[1]{>{\centering\arraybackslash}p{#1}}
\begin{table*}[!t]
\renewcommand{\arraystretch}{1.5}
\centering
\caption{ \small  Comparison with existing low-rank compression algorithms. ``Greedy'' indicates whether the algorithm employs greedy low-rank compression. ``Pre-train'' indicates whether the algorithm supports pre-training tasks in deep learning. ``Fine-tune'' indicates whether the algorithm supports fine-tuning tasks in deep learning.
}\label{table:alg_cmp}
\begin{threeparttable}
\begin{tabular}{lP{1.35cm}P{1.35cm}P{1.2cm}P{1.2cm}P{1.2cm}P{2.5cm}P{2.5cm}}
    \toprule
    \textbf{Algorithm} &
    \textbf{\renewcommand{\arraystretch}{1.0}\begin{tabular}[c]{@{}c@{}}\hspace{-1.5mm}Communication\\ \hspace{-1.5mm}Efficient\end{tabular}} &
    \textbf{\renewcommand{\arraystretch}{1.0}\begin{tabular}[c]{@{}c@{}}\hspace{1.5mm}Error\\ \hspace{1.5mm}Feedback\end{tabular}} &
    \textbf{Greedy} &
    \textbf{Pre-train} &
    \textbf{Fine-tune} &
    \textbf{\renewcommand{\arraystretch}{1.0}\begin{tabular}[c]{@{}c@{}}MSGD-Type\\ Convergence\end{tabular}} & \textbf{\renewcommand{\arraystretch}{1.0}\begin{tabular}[c]{@{}c@{}}Adam-Type\\ Convergence\end{tabular}}\\
    \midrule
    PowerSGD\cite{vogels2019powersgd} &\hspace{0.3cm}\cmark & \hspace{1mm}\cmark & \cmark & \cmark & \cmark & N. A. &  N. A. \\
    LoRA\cite{hu2022lora} &\hspace{0.3cm}\cmark & \hspace{1mm}\xmark & \cmark & \xmark & \cmark & N. A. & N. A.\\
    GaLore\cite{zhao2024galore} &\hspace{0.3cm}\cmark & \hspace{1mm}\xmark & \cmark & \cmark & \cmark & N. A. & N. A.\\
    GoLore$^{*}$\cite{he2024subspace} &\hspace{0.3cm}\cmark & \hspace{1mm}\xmark & \xmark & \cmark & \cmark & $\mathcal{O}(\frac{1}{T} + \frac{\sigma}{\sqrt{T}})$ & N. A.  \\
    LDAdam$^{\dagger}$\cite{robert2024ldadam} & \hspace{0.3cm}\xmark & \hspace{1mm}\cmark & \cmark & \cmark & \cmark & N. A. & $\mathcal{O}\left(\frac{1}{\sqrt{T}}+\frac{\sigma^2}{\sqrt{T}}\right)$\\
    SEPARATE\cite{zhao2025separate} & \hspace{0.3cm}\cmark& \hspace{1mm}\cmark & \xmark  & \cmark & \cmark & N. A. & $\mathcal{O}\left(\frac{1}{\sqrt{T}}+\frac{\sigma}{\sqrt{NT}}\right)$ \\
    \rowcolor{purple!15} \textbf{GreedyLore} & \hspace{0.3cm}\cmark & \hspace{1mm}\cmark & \cmark & \cmark & \cmark & $\boldsymbol{\mathcal{O}\left(\frac{1}{T^{2/3}}+\frac{\sigma}{\sqrt{NT}}\right)}$ & $\boldsymbol{\mathcal{O}\left(\frac{1}{T^{2/3}}+\frac{\sigma}{\sqrt{NT}}\right)}$\\
    \bottomrule
\end{tabular}
\begin{tablenotes}
    \footnotesize
    \item[$*$] The MSGD-type rate for GoLore has only been established in the single-node setting because no multi-node rate has been established yet. 
    \item[$\dagger$] LDAdam implicitly assumes that its projection introduces contractive errors, which is a more restrictive assumption.
\end{tablenotes}
\end{threeparttable}
\vspace{-4mm}
\end{table*}

\subsection{Limitations in Existing Literature}
\label{limitation}
The effectiveness of low-rank compression lies in projecting the transmitted matrix onto a low-rank subspace while preserving as much information as possible. A natural approach is to apply Singular Value Decomposition (SVD) and project the matrix onto the subspace spanned by its top-$r$ singular vectors, thereby capturing its most significant components. To reduce the high per-iteration computational cost of SVD, GaLore\footnote{GaLore~\cite{zhao2024galore} was originally proposed as a memory-efficient algorithm in the single-node setting. However, it can be naturally extended to a communication-efficient algorithm in distributed settings; see Appendix \ref{sec:dist_galore}.}\cite{zhao2024galore} and its variants\cite{chen2024fira, zhang2024q} update the low-rank subspace lazily—that is, they perform SVD periodically (e.g., once every $\tau$ iterations) and reuse the same subspace within the same period. Another well-known algorithm, PowerSGD~\cite{vogels2019powersgd}, employs orthogonal operations in power iteration to gradually align the projection matrices with the optimal rank-$r$ approximation of the gradients over time, reducing computational overhead without explicit SVD. We refer to these SVD-based approaches as \textbf{Greedy Low-Rank Compression}, as they aim to preserve the most informative structure of the transmitted matrix. Although these greedy methods exhibit strong empirical performance, they \textbf{{lack convergence guarantees}} in stochastic optimization settings. For instance, recent work~\cite{he2024subspace} shows that when gradient noise dominates the true gradient in stochastic optimization, the greedy SVD-derived subspace used in GaLore captures primarily the noise component, ultimately leading to non-convergence.

\begin{figure}[H]
\vspace{-1mm}
\centering
\includegraphics[width=2.5in]{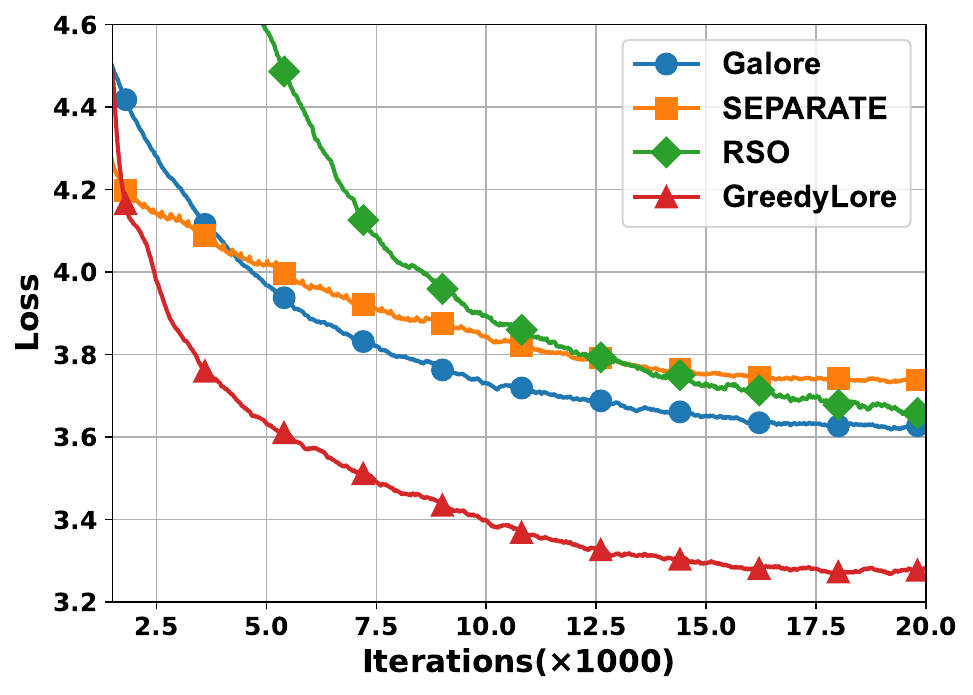}
\caption{\small Loss curves for pre-training LLaMA-130M using random projection and greedy projection as low-rank compressors with rank 32. Compression is applied starting from step 1,000. It can be observed that employing random projection in the early training phase introduces excessive noise, which leads to slower convergence.}
\label{fig:rand}
\vspace{-3mm}
\end{figure}

In contrast to greedy methods, \textbf{Random Low-Rank Compression} projects the transmitted matrix onto a randomly selected low-rank subspace, independent of the current gradient. By avoiding alignment with the dominated directions, random projections preserve unbiased gradient information in expectation, even when gradient noise dominates. As a result, these methods are robust to noise and enjoy strong convergence guarantees~\cite{he2024subspace, zhao2025separate}. However, because they do not exploit the underlying low-rank structure to capture the dominant components of the transmitted matrix—particularly when the true gradient dominates the noise—random compression often incurs significant compression error, leading to \textbf{slower convergence in practice}. 

The above discussion on greedy and random low-rank compression motivates the following fundamental open question:


\begin{center}
\setlength\fboxsep{4pt}
\colorbox{gray!20}{\parbox{\dimexpr\linewidth-2\fboxsep\relax}{
\textit{\textbf{{\em (Question)}} Can we develop distributed 
 communication efficient methods based on \textbf{greedy low-rank compression} that achieve both \textbf{strong empirical performance} and \textbf{theoretical convergence guarantees} under practical settings?}}}
\end{center}
By practical settings, we refer to stochastic optimization using MSGD or Adam optimizer which is common in deep learning.

\vspace{-2mm}

\subsection{Main Challenges} 
\label{sec:challenge}
Addressing this open question involves several key challenges, which call for novel algorithmic developments.
\begin{itemize}
\item[\textbf{C1.}] \textbf{Greedy compression.} When gradient noise dominates the true gradient in the stochastic gradient, greedy SVD-based low-rank compression may fail to capture meaningful gradient information, potentially leading to non-convergence; see the detailed discussion in~\cite{he2024subspace}.

\vspace{1mm}
\item[\textbf{C2.}] \textbf{Lazy subspace update.} To reduce the computational overhead of SVD, it is common to perform SVD periodically and reuse the same subspace within each period. However, this lazy subspace strategy may result in a non-contractive gradient compressor—i.e., the compression error does not vanish as the gradient approaches zero—potentially leading to non-convergence. Moreover, such lazy updates can nullify error feedback, an advanced technique designed to compensate for compression errors.

\vspace{1mm}
\item[\textbf{C3.}] \textbf{Capturing global low-rank structure.} To ensure strong empirical performance, an effective low-rank gradient compressor is desired to capture the low-rank structure of global gradient. However, since each computing node only has access to its local gradient, simply aggregating compressed local gradients fails to preserve the true structure of global gradient, as shown in \cite{wang2018atomo, lin2017deep, karimireddy2019error}. 
\end{itemize}

\vspace{-1.5mm}
\subsection{Contributions} 
By addressing the above challenges with novel solutions, this paper provides affirmative answers to the open question.
\begin{itemize}[leftmargin=1.5em]
\item We propose a novel \underline{\textbf{Greedy}} \underline{\textbf{Lo}}w-\underline{\textbf{R}}ank Gradi\underline{\textbf{e}}nt Compression algorithm (termed \textbf{GreedyLore}) for distributed learning. GreedyLore has three key components: (1) it incorporates error feedback to correct the bias introduced by greedy compression; (2) it introduces a semi-lazy subspace update that ensures a contractive compressor throughout all iterations and preserves the effectiveness of  error feedback; and (3) it proposes an approximate global top-$r$ projection technique that enables each node to locally estimate the low-rank structure of the global gradient. These innovations address the challenges in Section~\ref{sec:challenge}, enabling strong performance with convergence guarantees.

\vspace{1mm}
\item Under standard assumptions, we establish  that GreedyLore converges at a rate of $\mathcal{O}\left(\frac{\sigma}{\sqrt{NT}} + \frac{1}{T}\right)$ with momentum SGD (MSGD) and Adam optimizer, where $\sigma$ denotes the gradient noise and $T$ is the number of iterations—matching the convergence rate of standard distributed  algorithms without any communication compression. To the best of our knowledge, this is the first convergence result for greedy low-rank communication compression. Moreover, it is also the first to achieve linear speedup for low-rank communication compression—whether random or greedy—when using MSGD and Adam (see Table \ref{table:alg_cmp}).

\vspace{1mm}
\item We efficiently implement GreedyLore algorithm with PyTorch’s communication hook mechanism. GreedyLore can be seamlessly integrated with first-order optimizers such as MSGD and Adam, and is compatible with system-level implementations like Distributed Data Parallel (DDP). Extensive experiments (e.g., Figure~\ref{fig:os}) demonstrate that GreedyLore achieves superior  performance compared to other low-rank communication compression algorithms.
\end{itemize}

\section{Related Works}
\subsection{Communication Compression}
Communication compression is a key technique for reducing communication overhead in distributed learning. The pioneering work QSGD~\cite{alistarh2017qsgd} introduced a systematic approach to gradient quantization. Since then, a variety of compression methods have been proposed, including sign-based compression~\cite{bernstein2018signsgd}, natural compression~\cite{horvoth2022natural}, and low-rank compression~\cite{wang2018atomo, vogels2019powersgd}. To mitigate the adverse effects of compression error, several studies~\cite{seide20141, karimireddy2019error, mishchenko2024distributed, li2020acceleration} introduced error feedback, which accumulates compression errors and incorporates them into future gradient updates, thereby preserving more informative signal. EF21~\cite{richtarik2021ef21} extends this idea by maintaining a local gradient tracker for each worker, effectively alleviating the impact of data heterogeneity. NEOLITHIC and its variants~\cite{huang2022lower, he2023lower, he2023unbiased} established lower bounds for distributed learning with communication compression. Notably, EF21-MSGD~\cite{fatkhullin2023momentum} matches these bounds and achieves optimal convergence rates. However, none of these works can establish convergence guarantees for greedy low-rank communication compression with lazy SVD updates. 

\subsection{Low-Rank Gradient Compression} 
Extensive literature has shown that gradients generated in large deep neural networks—such as large language models (LLMs)—exhibit a low-rank structure~\cite{chen2024enhancing, zhao2024galore, malladi2023fine}. Leveraging this property, various studies have developed memory-efficient algorithms for pre-training and fine-tuning LLMs. LoRA~\cite{hu2022lora} capitalizes on the low-rank nature of incremental updates to reduce memory overhead during adaptation. Subspace optimization approaches such as GaLore~\cite{zhao2024galore} and its variants~\cite{chen2024fira,zhang2024q} project the gradient onto a greedy SVD-derived subspace to reduce optimizer memory usage during pre-training. However, as demonstrated in~\cite{he2024subspace}, when the algorithm approaches a local minimum and the true gradient vanishes while the noise persists, the selected subspace increasingly captures only the noise component, ultimately leading to non-convergence. 

To overcome this limitation, a complementary line of work explores random low-rank gradient compression~\cite{he2024subspace,chen2025memory,hao2024flora,firtina2020apollo}, which can preserve useful gradient information even when it is dominated by noise.  GoLore~\cite{he2024subspace} and RSO~\cite{chen2025memory} provide rigorous convergence guarantees for this class of algorithms. A more recent approach, LDAdam~\cite{robert2024ldadam}, combines error feedback with large-batch gradient accumulation to establish convergence guarantees for greedy low-rank compression. However, these guarantees rest on a restrictive contractiveness assumption on the compression error, which may not hold in practice. Moreover, LDAdam is not suitable for communication-efficient distributed settings, as it requires full gradient communication at each iteration. In contrast, our proposed GreedyLore addresses both theoretical and practical challenges through novel algorithmic designs; see Section~\ref{sec:challenge}.

The low-rank structure of gradients can also help reduce communication overhead in distributed learning. PowerSGD~\cite{vogels2019powersgd} employs power iterations to identify an effective low-rank subspace, enabling the compressed gradients to progressively converge to an optimal rank-$r$ approximation over successive iterations. Building upon this framework, decentralized algorithms such as PowerGossip directly compress model differences using low-rank linear compressors, enabling communication efficiency over arbitrary network topologies\cite{vogels2020powergossip}. To ensure compatibility with system-level implementations, ACP-SGD alternately compresses and communicates low-rank matrices with provable convergence and minimal overhead~\cite{zhang2023evaluation}. In the federated learning setting, Riemannian Low-Rank Model Compression~\cite{xue2023riemannian} leverages manifold optimization to update local models with convergence guarantees. Although originally proposed as a memory-efficient approach, GaLore~\cite{zhao2024galore} can also be naturally extended for communication reduction by transmitting compressed low-rank gradients. Recent works, such as SEPARATE~\cite{zhao2025separate} and RSO~\cite{chen2025memory}, address this issue by providing rigorous convergence guarantees for random low-rank communication compression. Nevertheless, none of these methods have addressed the fundamental question highlighted in Section~\ref{limitation}. 

\section{Preliminary}
\noindent \textbf{Notation.} Let $\mathds{1}_n$ denote the vector of all-ones of $n$ dimensions and $\bm{I}_n \in \mathbb{R}^{n\times n}$ the identity matrix. We introduce the set $[n] := \{1,\cdots, n\}$. Given the index set $\mathcal{J}$, notation $\bm{U}_{:, \mathcal{J}}$ refers to the submatrix of $\bm{U}$ formed by selecting the columns of $\bm{U}$ indexed by the set $\mathcal{J}$. The notation $\bm{U}_{:,:r}$ refers to the first $r$ columns of the matrix $\bm{U}$.

\vspace{1mm}
\noindent \textbf{Contractive compressor.} Contractive compressors are widely used in communication compression, with Top-$K$ being a representative example~\cite{richtarik2021ef21, huang2022lower}. Their key property is that the compression error diminishes as the original variable being compressed approaches zero. A formal definition follows:
\begin{definition}[Contractive Compressor]
\label{compressor-assumption}
A compressor $\mathcal{C}(\cdot)$ is defined as a contractive compressor if it satisfies
\begin{equation*}
    \mathbb{E}_{\mathcal{C}}\Big[ \|\mathcal{C}(\bm{X})-\bm{X}\|_F^2 \Big]\leq (1-\delta) \|\bm{X}\|_F^2,\quad \forall \bm{X}\in\mathbb{R}^{m\times n},
\end{equation*}
where $\delta\in(0,1]$ is the contractive factor. The expectation is taken over the randomness of the compression operator $\mathcal{C}$.
\end{definition}
The contractive property of a compressor is essential for ensuring the convergence of algorithms that rely on it. While SVD-based greedy low-rank compression satisfies the contractive property, its lazy variant—where SVD is performed periodically (e.g., once every $\tau$ iterations) and the same subspace is reused within each period—does not (see Proposition~\ref{prop-not-contractive}). This lack of contractiveness presents a fundamental challenge for convergence analysis of greedy low-rank compression.

\vspace{1mm}
\noindent \textbf{Error feedback.} Consider the unconstrained problem: 
\begin{align}
\label{eq-unconstrained-prob}
\min_{\bm{X} \in \mathbb{R}^{m\times n}}\ f(\bm{X}).
\end{align}
The error feedback technique proposed in \cite{seide20141} for solving the above unconstrained optimization problem is
\begin{subequations}
\begin{align}
    \bm{X}_{t+1}
      &= \bm{X}_{t}
         - \gamma\,\mathcal{C}\bigl(\nabla f(\bm{X}_{t}) + \bm{E}_{t-1}\bigr),
      \label{eq-ef-1}\\
    \bm{E}_{t}
      &= \nabla f(\bm{X}_{t}) + \bm{E}_{t-1}
        - \mathcal{C}\bigl(\nabla f(\bm{X}_{t}) + \bm{E}_{t-1}\bigr),
      \label{eq-ef-2}
\end{align}
\end{subequations}
with initialization $\bm{E}_{-1} = \bm{0}$, where $\mathcal{C}(\cdot)$ denotes a contractive compressor and $\bm{E}_t$ accumulates the compression error over time. This method compensates for the distortion introduced by lossy gradient compression. In particular, the residual error $\bm{E}_t$ captures the information lost at each iteration and reincorporates it into subsequent updates. As a result, error feedback effectively mitigates the bias introduced by compression and preserves convergence guarantees—even under aggressive communication constraints.

\vspace{1mm}
\noindent \textbf{MSGD optimizer.} Momentum stochastic gradient descent (MSGD) enhances standard SGD by maintaining an exponential moving average of past gradients, which helps dampen oscillations and accelerates convergence towards local minima. Let $\bm{G}_t$ denote the stochastic gradient of the global loss function in \eqref{eq:formulation}. The MSGD update is given by:
\begin{subequations}
\label{eq-msgd}
\begin{align}
    \bm{M}_t &= \beta\bm{M}_{t-1} + (1-\beta)\bm{G}_t, \label{eq-msgd-1}\\
    \bm{X}_{t+1} &= \bm{X}_{t} - \gamma\bm{M}_t, \label{eq-msgd-2}
\end{align}
\end{subequations}
where $\beta\in[0,1)$ is the momentum coefficient and $\gamma$ is the learning rate. For initialization, we let $\bm{M}_{-1} = \bm{0}$. 

\vspace{1mm}
\noindent \textbf{Adam optimizer.} Adam is one of the most widely used optimizers in deep neural networks, particularly in large language models. Let $\bm{G}_t$ be the stochastic gradient of the global loss function in \eqref{eq:formulation}, Adam will update as follows:
\begin{subequations}
\label{eq-adam}
\begin{align}
    \bm{M}_t & = \beta_1 \bm{M}_{t-1} + (1-\beta_1) \bm{G}_t, \label{eq-adam-1}  \\
    \bm{V}_t & = \beta_2 \bm{V}_{t-1} + (1-\beta_2)\bm{G}_t\odot\bm{G}_t, \label{eq-adam-2} \\
    \bm{X}_{t+1} &= \bm{X}_{t} - \frac{\gamma}{\sqrt{\bm{V}_t} + \epsilon} \odot \bm{M}_t, \label{eq-adam-3}
\end{align}
\end{subequations}
where $\beta_1 \in [0,1)$ and $\beta_2 \in [0,1)$ are momentum coefficients, $\gamma$ is the learning rate, and $\odot$ denotes the elementwise product. For initialization, we have $\bm{M}_{-1} = \bm{0}$ and $\bm{V}_{-1} = \bm{0}$. 

\section{Basic low-rank compression framework}
\label{sec-prelim-framework}
We begin with a preliminary framework for low-rank gradient compression, which serves as the foundation for our proposed GreedyLore algorithm. 
Let $\bm{G}_t^{(i)} = \nabla F_i(\bm{X}_t;\bm{\xi}_t^{(i)})$ and $\bm{G}_t = \frac{1}{N}\sum_{i=1}^N \bm{G}_t^{(i)}$, where 
$t$ denotes the iteration index and $i$ indexes the computing node. The framework is:
\begin{center}
\setlength\fboxsep{4pt}
\colorbox{gray!20}{\parbox{\dimexpr\linewidth-2\fboxsep\relax}{
\begin{subequations}
\label{eq-prelim-framework}
\begin{align}
    \bm{P}_t & = \mathsf{\bf Lazy}\mbox{-}\mathsf{\bf SVD}(\bm{G}_t, \bm{P}_{t-1}, t), \label{eq-prelim-framework-1}  \\
    \bm{R}^{(i)}_t &= \bm{P}^\top_t\bm{G}_t^{(i)}, \ \bm{R}_t = \frac{1}{N}\sum_{i=1}^N \bm{R}^{(i)}_t, \ \hat{\bm{G}}_t = \bm{P}_t \bm{R}_t, \label{eq-prelim-framework-2} \\
    \bm{X}_{t+1} &= \mathsf{\bf Optimizer}(\bm{X}_{t}, \hat{\bm{G}}_t, \gamma). \label{eq-prelim-framework-3}
\end{align}
\end{subequations}}}
\end{center}

In this framework, step \eqref{eq-prelim-framework-1} identifies the projection matrix $\bm{P}_t \in \mathbb{R}^{m\times r}$ using  $\mathsf{\bf Lazy}\mbox{-}\mathsf{\bf SVD}$ operator with period $\tau$, i.e.,
\begin{align}
\label{eq-SVD-P}
    \bm{U},\bm{\Sigma}, \bm{V} = \mathrm{\bf SVD}(\bm{G}_t), \quad \bm{P}_t = \bm{U}_{:,:r}\in \mathbb{R}^{m\times r},
\end{align}
when $\mathsf{mod}(t,\tau) = 0$ otherwise $\bm{P}_t = \bm{P}_{t-1}$. Notation $\bm{U}_{:,:r}$ denotes the first $r$ columns of the matrix $\bm{U}$. Since $\mathsf{\bf Lazy}\mbox{-}\mathsf{\bf SVD}$ performs the SVD operation only once every $\tau$ iterations, it significantly reduces computational overhead. In step \eqref{eq-prelim-framework-2}, each computing node compresses its local gradient $\bm{G}_t^{(i)} \in \mathbb{R}^{m \times n}$ into a low-dimensional representation $\bm{R}_t^{(i)} \in \mathbb{R}^{r \times n}$, communicates it via the All-Reduce protocol to compute the global average $\bm{R}_t$, and then reconstructs a high-dimensional approximate gradient $\hat{\bm{G}}_t \in \mathbb{R}^{m \times n}$. 
It is worth noting that the full-dimensional communication of ${\bm G}_t^{(i)}$ in \eqref{eq-SVD-P} occurs only once every $\tau$ iterations; in the other iterations, only the low-rank matrix ${\bm R}_t^{(i)}$ needs to be transmitted. In step \eqref{eq-prelim-framework-3}, each computing node updates $\bm{X}_t$ using the optimizer such as MSGD and Adam, where $\gamma$ denotes the learning rate. Furthermore, since each node obtains the globally averaged gradient \(\bm{G}_t\) when \(\mathsf{mod}(t, \tau) = 0\) in $\mathsf{\bf Lazy}\mbox{-}\mathsf{\bf SVD}$, it can be directly used in the optimizer instead of the approximate \(\hat{\bm{G}}_t\). 

The implementation details are presented in Algorithm~\ref{alg:prelimnary}. All-Reduce is a collective communication operation widely used in distributed computing. In this process, each node sends its local gradient estimate $\bm{G}_t^{(i)}$ (or $\bm{R}_t^{(i)}$) and receives the global average $\bm{G}_t$ (or $\bm{R}_t$). There are several variants of All-Reduce. In high-performance data center clusters with high-bandwidth GPU interconnects, All-Reduce is typically implemented using Ring-All-Reduce~\cite{patarasuk2009bandwidth}. In contrast, for communication-constrained settings such as federated learning, it is often implemented via a parameter server~\cite{li2014communication}.

Framework~\eqref{eq-prelim-framework} faces two key challenges, including greedy low-rank compression and infrequent (lazy) subspace updates. Addressing these issues motivates the development of our proposed GreedyLore algorithm.

\section{GreedyLore Algorithm Development} 

\subsection{Error feedback}
\noindent \textbf{Low-rank compression with error feedback}. To address the bias brought by greedy low-rank compression, we integrate \textit{error feedback (EF)}~\cite{richtarik2021ef21,fatkhullin2023momentum} to the algorithm, which accumulates the true gradient over time, enabling it to eventually outweigh the noise and become the dominant component of the stochastic gradient. To begin with, we write \eqref{eq-prelim-framework-2} as follows:
\begin{align}\label{2nasdn}
\hat{\bm{G}}_t = \frac{1}{N}\sum_{i=1}^N \mathcal{C}_t(\bm{G}_t^{(i)})\ \mbox{where}\ \mathcal{C}_t(\bm{G}_t^{(i)}) := \bm{P}_t\bm{P}_t^\top \bm{G}_t^{(i)}.
\end{align}
To incorporate EF to accumulate useful gradient, we propose
\begin{subequations}
\label{eq-ef}
\begin{align}
    \hat{\bm{G}}_t^{(i)} &= \mathcal{C}_t(\bm{G}_t^{(i)} + \bm{E}_{t-1}^{(i)}),\label{eq:ef1} \\
    \bm{E}_{t}^{(i)} &= \bm{G}_t^{(i)} + \bm{E}_{t-1}^{(i)} - \hat{\bm{G}}_t^{(i)},\label{eq:ef2}
\end{align}
\end{subequations}
and update $\hat{\bm{G}}_t = \frac{1}{N}\sum_{i=1}^N \hat{\bm{G}}_t^{(i)}$. The auxiliary variable $\bm{E}_t^{(i)}$ accumulates the compression error and is added back to the gradient $\bm{G}_t^{(i)}$ in subsequent iterations to compensate for it. We initialize $\bm{E}_{-1}^{(i)} = \bm{0}$ for all nodes $i$. 

\vspace{1mm}
\noindent \textbf{Implementation.} With the compressor $\mathcal{C}_t(\cdot)$ defined in \eqref{2nasdn}, we can implement \eqref{eq:ef1} in a communication-efficient manner: 
\begin{align}\label{8zhb2we09}
\hspace{-1mm}\bm{R}^{(i)}_t \hspace{-0.5mm}=\hspace{-0.5mm} \bm{P}^\top_t(\bm{G}_t^{(i)} \hspace{-0.5mm}+\hspace{-0.5mm} \bm{E}_{t\hspace{-0.5mm}-\hspace{-0.5mm}1}^{(i)}),  \bm{R}_t \hspace{-0.5mm}=\hspace{-0.5mm} \frac{1}{N}\sum_{i=1}^N \bm{R}^{(i)}_t,  \hat{\bm{G}}_t \hspace{-0.5mm}=\hspace{-0.5mm} \bm{P}_t \bm{R}_t,
\end{align}
where $\bm{R}^{(i)}_t$ is the low-rank variable to be communicated.

\begin{algorithm}[!t]

\caption{Basic low-rank compression framework}
\label{alg:prelimnary}
\SetKwFunction{lazysvd}{Lazy-SVD}
\SetKwProg{sub}{Subroutine}{}{}

\vspace{1pt}
\JustifyAlgo{\hspace{-4.95mm} \textbf{Input}: $N$ nodes, learning rate $\gamma$, number of total iterations $T$, subspace changing frequency $\tau$, rank $r$, error buffer $\bm{E}_{-1} = \bm{0}$, weight $\bm{X}_{0} \in \mathbb{R}^{m \times n}$ and projection matrix $\bm{P}_{-1} \in \mathbb{R}^{m \times r}$ on each node $i \in [N]$ with shape $m \le n$. }\\
\vspace{2pt}
\textbf{Output}: Sequence of model weights $\{\bm{X}_t\}_{t=0}^{T+1}$.

\For{$t=0,\ldots,T$}{
    \textbf{(On $i$-th node)} \\
    $\bm{G}_t^{(i)} \gets \nabla F_i({\bm{X}}_t; {\bm{\xi}}_t^{(i)})$ with local data ${\bm{\xi}}_t^{(i)}$. \\
    $\bm{P}_t, \bm{G}_t \gets \mathsf{\bf Lazy}\mbox{-}\mathsf{\bf SVD}(\{\bm{G}_t^{(i)}\}_{i=1}^N, \bm{P}_{t-1}, t)$. \\
    $\bm{R}_t^{(i)} \gets {\bm{P}_t}^\top{\bm{G}}_t^{(i)}$. \\
    ${\bm{R}}_t \gets \frac{1}{N}\sum_{i=1}^{N}{\bm{R}}_t^{(i)}$. \hfill (All-Reduce)   \\
    $ \hat{\bm{G}}_t \gets \bm{P}_t\bm{R}_t\ \mathsf{\bf if}\  \mathsf{mod}(t, \tau) \ne 0$ \textbf{otherwise} $\bm{G}_t$.\\
  
    $\bm{X}_{t+1} \gets \mathsf{\bf Optimizer}(\bm{X}_{t}, \hat{\bm{G}}_t, \gamma)$. \\
    }
    \Return $\{\bm{X}_{t}\}_{t=0}^{T+1}$.

\vspace{2mm}
\sub{$\mathsf{\bf Lazy}\mbox{-}\mathsf{\bf SVD}(\{\bm{G}_t^{(i)}\}_{i=1}^N, \bm{P}_{t-1}, t)$}{
\uIf{$\mathsf{mod}(t, \tau) = 0$}{
    $\bm{G}_t \gets \frac{1}{N}\sum_{i=1}^{N}{\bm{G}_t}^{(i)} $.\hfill (All-Reduce) \\
    $\bm{U},\bm{\Sigma}, \bm{V} \gets \mathrm{\bf SVD}(\bm{G}_t)$. \\
    \vspace{-0.5mm}
    \Return $\bm{U}_{:,:r}, \bm{G}_t$.
}
\Else{
    \Return $\bm{P}_{t-1}, \bm{0}$.
}
}

\end{algorithm}

\vspace{-3mm}
\subsection{Semi-lazy SVD update}
While the error feedback update~\eqref{8zhb2we09} is effective in correcting the bias introduced by greedy low-rank compression in stochastic settings, its effectiveness relies on updating the projection matrix \(\bm{P}_t\) at every iteration. However, in~\eqref{8zhb2we09}, the projection matrix \(\bm{P}_t\) is generated by the \(\mathsf{\bf Lazy}\mbox{-}\mathsf{\bf SVD}\) operator, which enforces \(\bm{P}_t = \bm{P}_{t-1}\) for all iterations within a period of \(\tau\) steps. This section demonstrates that the \(\mathsf{\bf Lazy}\mbox{-}\mathsf{\bf SVD}\) operator is incompatible with error feedback, and that a modification is necessary for error feedback to remain effective.

\noindent \textbf{$\mathsf{\bf Lazy}\mbox{-}\mathsf{\bf SVD}$ nullifies error feedback.} 
The following proposition shows that when projection \(\bm{P}\) remains constant, the effect of error feedback \textit{vanishes}—i.e., the compressed gradient remains the \textit{same} with or without error feedback.

\begin{figure}[H]
\centering
\includegraphics[width=2.5in]{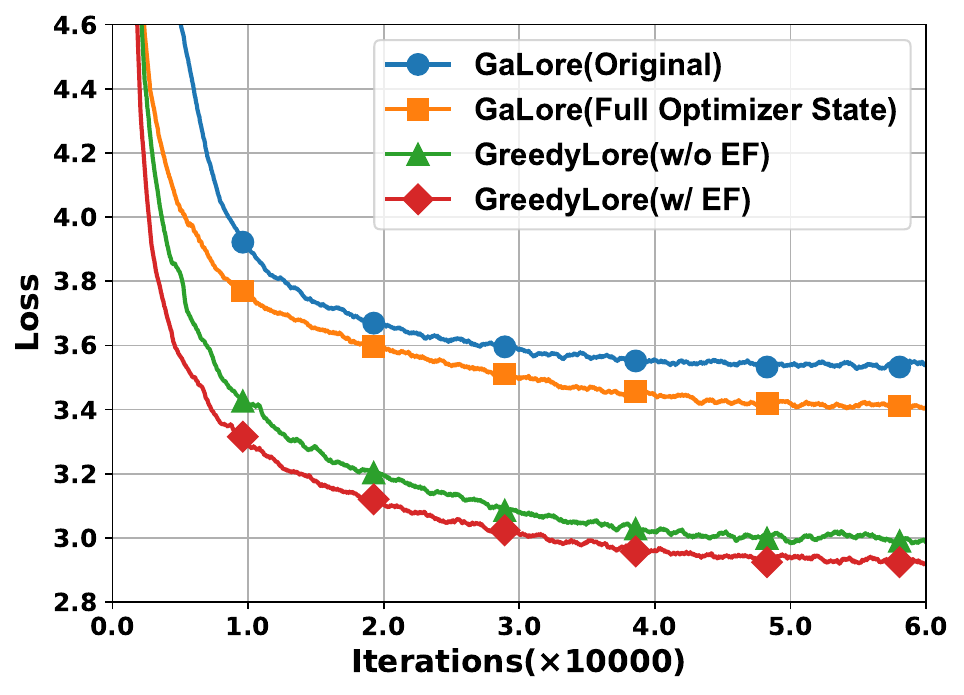}
\caption{Loss curves for pre-training LLaMA-350M using the original GaLore, GaLore with full optimizer states, and GreedyLore with and without error feedback, under a compression rank $r=32$.}
\label{fig:os}
\end{figure}

\begin{center}
\setlength\fboxsep{4pt}
\colorbox{gray!20}{\parbox{\dimexpr\linewidth-2\fboxsep\relax}{
\begin{proposition}
Consider compressor $\mathcal{C}_t(\bm{G}_t) = \bm{P} \bm{P}^\top \bm{G}_t$ for $t = 1, \ldots, \tau - 1$, where $\bm{G}_t \in \mathbb{R}^{m \times n}$ is an arbitrary matrix and $\bm{P} \in \mathbb{R}^{m \times r}$ is a fixed projection matrix satisfying $\bm{P}^\top \bm{P} = \bm{I}_r$. With error feedback update \eqref{eq-ef}, it follows that
\begin{align}\label{eq:lazy-svd-nullify}
\mathcal{C}_t(\bm{G}_t^{(i)} + \bm{E}_{t-1}^{(i)}) = \mathcal{C}_t(\bm{G}_t^{(i)}), \ \ \forall t=1,\cdots, \tau-1.
\end{align}
\end{proposition}}}
\end{center}

\begin{proof} Due to the linearity of the compressor, we have
\begin{align}\label{3247uads}
\mathcal{C}_t(\bm{G}_t^{(i)} + \bm{E}_{t-1}^{(i)}) = \mathcal{C}_t(\bm{G}_t^{(i)}) + \mathcal{C}_t(\bm{E}_{t-1}^{(i)}).
\end{align}
We next demonstrate $\mathcal{C}_t(\bm{E}_{t-1}^{(i)}) = \bm{0}$ with constant $\bm{P}$. When $t=1$, \eqref{3247uads} is trivial. When $t>1$,
\begin{align}
\mathcal{C}_t(\bm{E}_{t-1}^{(i)}) &=\bm{P}\bm{P}^\top\bm{E}_{t-1}^{(i)} \nonumber\\
&\overset{\eqref{eq-ef}}{=}\bm{P}\bm{P}^\top(\bm{G}_{t-1}^{(i)} + \bm{E}_{t-2}^{(i)}-\mathcal{C}_{t-1}(\bm{G}_{t-1}^{(i)} + \bm{E}_{t-2}^{(i)})) \nonumber\\
&=\bm{P}\bm{P}^\top(\bm{I}_n - \bm{P}\bm{P}^\top)(\bm{G}_{t-1}^{(i)} + \bm{E}_{t-2}^{(i)} ) \nonumber\\
&=(\bm{P}\bm{P}^\top - \bm{P}\bm{I}_r\bm{P}^\top)(\bm{G}_{t-1}^{(i)} + \bm{E}_{t-2}^{(i)}) = \bm{0}. \label{238hax0}
\end{align}
With \eqref{3247uads} and \eqref{238hax0}, we achieve the result in \eqref{eq:lazy-svd-nullify}. 
\end{proof}

\noindent \textbf{$\mathsf{\bf Lazy}\text{-}\mathsf{\bf SVD}$ induces a non‐contractive compressor.} As established in the existing literature~\cite{richtarik2021ef21,huang2022lower}, the convergence of the error feedback update~\eqref{eq-ef} relies on the use of a contractive compressor (see Definition~\ref{compressor-assumption}). The following proposition demonstrates that \(\mathcal{C}_t(\bm{G}_t) = \bm{P}_t \bm{P}_t^\top \bm{G}_t\) may fail to be contractive when \(\bm{P}_t\) is generated using the \(\mathsf{\bf Lazy}\text{-}\mathsf{\bf SVD}\) operator, which ensures that \(\bm{P}_t = \bm{P}_{t-1}\) for all iterations within a period of \(\tau\) steps.
\begin{center}
\setlength\fboxsep{4pt}
\colorbox{gray!20}{\parbox{\dimexpr\linewidth-2\fboxsep\relax}{
\begin{proposition}
\label{prop-not-contractive}
Consider the compressor $\mathcal{C}(\bm{G}_t) = \bm{P} \bm{P}^\top \bm{G}_t$ for $t = 1, \ldots, \tau - 1$, where 
$\bm{U},\bm{\Sigma}, \bm{V} = \mathrm{SVD}(\bm{G}_0)$ and $\bm{P} = \bm{U}_{:,:r}\in \mathbb{R}^{m\times r}$. There exists some $\bm{G}_t$ such that 
\begin{align}
\|\mathcal{C}(\bm{G}_t) - \bm{G}_t\|^2 = \|\bm{G}_t\|^2,
\end{align}
which implies that $\mathcal{C}(\bm{G}_t)$ is a non-contractive compressor.
\end{proposition}}}
\end{center}
\begin{proof}
We consider a simple scenario in which 
\[
\bm{G}_0 = \begin{pmatrix}2 & 0\\0 & 1\end{pmatrix} \quad \mbox{with rank-$1$ projection} \quad 
\bm{P} = \begin{pmatrix}1\\0\end{pmatrix}.
\]
We further consider 
\[
\bm{G}_{t} = \begin{pmatrix}0 & 0\\0 & 1\end{pmatrix}, \ \mbox{and hence} \ \mathcal{C}(\bm{G}_{t}) = \bm{P}\,\bm{P}^\top\,\bm{G}_{t} = \begin{pmatrix}0 & 0\\0 & 0\end{pmatrix}.
\]
In this scenario, it holds that $\|\mathcal{C}(\bm{G}_t) - \bm{G}_t\|^2 = \|\bm{G}_t\|^2$. 
\end{proof}
\noindent The above proposition implies that the compressor becomes non-contractive when \({\bm G}_t\) lies in a subspace orthogonal to \(\bm{P}\), which can occur in certain optimization problems, as illustrated in Appendix~\ref{sec:expaned_counter_example}. 

\vspace{1mm}
\noindent \textbf{\(\mathsf{\bf Semi}\mbox{-}\mathsf{\bf Lazy}\ \mathsf{\bf SVD}\) operator.} As discussed above, maintaining a constant projection matrix \(\bm{P}\) over a period nullifies the effect of error feedback and results in a non-contractive compressor. This suggests that algorithm design should incorporate a dynamic projection \(\bm{P}_t\) that adapts to the time-varying stochastic gradient \(\bm{G}_t\).
Similar to the \(\mathsf{\bf Lazy}\text{-}\mathsf{\bf SVD}\) operator, we propose to perform the SVD lazily every \(\tau\) iterations:
\begin{align}
\label{eq-SVD-P1}
\bm{U}, \bm{\Sigma}, \bm{V} = \mathsf{\bf SVD}(\bm{G}_t), \quad \bm{P}_t = \bm{U}_{:,:r}
,\end{align}
when \(\mathsf{mod}(t, \tau) = 0\). Otherwise, we update \(\bm{P}_t\) by solving the following subproblem:
\begin{align}
\label{eq:new_sub_problem}
\bm{P}_t = \argmin_{\bm{P} \in \mathcal{S}(\bm{U}, r)} \left\| \bm{G}_t - \bm{P} \bm{P}^\top \bm{G}_t \right\|_F^2,
\end{align}
where \(\mathcal{S}(\bm{U}, r)\) is defined as:
\[
\mathcal{S}(\bm{U}, r) := \left\{ \bm{U}_{:, \mathcal{J}} \in \mathbb{R}^{m \times r} \mid \mathcal{J} \subseteq \{1, \dots, n\},\ |\mathcal{J}| = r \right\}.
\]
In other words, \({\bm P}_t\) is updated by optimally selecting \(r\) columns from the given orthonormal matrix \(\bm{U}\) to minimize the compression error defined in~\eqref{eq:new_sub_problem}. We refer to equations~\eqref{eq-SVD-P1}–\eqref{eq:new_sub_problem} as the \textbf{\(\mathsf{\bf Semi}\mbox{-}\mathsf{\bf Lazy}\ \mathsf{\bf SVD}\)} operator. On the one hand, it performs SVD only once every \(\tau\) iterations; on the other hand, it dynamically updates \({\bm P}_t\) to track the time-varying \({\bm G}_t\). 

\vspace{1mm}
\noindent \textbf{\(\mathsf{\bf Semi}\mbox{-}\mathsf{\bf Lazy}\ \mathsf{\bf SVD}\) enables the contractive compressor.}
The following proposition demonstrates that \textbf{\(\mathsf{\bf Semi}\mbox{-}\mathsf{\bf Lazy}\ \mathsf{\bf SVD}\)} enables the contractive compressor.
\begin{center}
\setlength\fboxsep{4pt}
\colorbox{gray!20}{\parbox{\dimexpr\linewidth-2\fboxsep\relax}{
    \begin{proposition}\label{prop:contractive}
      Consider compressor $\mathcal{C}_t({\bm G}_t) = {\bm P}_t{\bm P}_t^\top {\bm G}_t$ where $\bm{G}_t \in \mathbb{R}^{m \times n}$ is an arbitrary matrix and $\bm{P}_t \in \mathbb{R}^{m \times r}$ is generated through \textbf{\(\mathsf{\bf Semi}\mbox{-}\mathsf{\bf Lazy}\ \mathsf{\bf SVD}\)} operator \eqref{eq-SVD-P1}–\eqref{eq:new_sub_problem}. It holds for any $\bm{G}_t \in \mathbb{R}^{m \times n}$ that 
      \begin{align}\label{eq:contractive}
        \bigl\|\bm{G}_t-\mathcal{C}_t(\bm{G}_t)\bigr\|_F^2 
        \;\le\;\Bigl(1-\frac{r}{m}\Bigr)\|\bm{G}_t\|_F^2. 
      \end{align}
    \end{proposition}
  }
}
\end{center}
\begin{proof} 
Without loss of generality, we assume \( t \in \{k\tau, k\tau+1, \ldots, (k+1)\tau - 1\} \). Let \({\bm U}\) denote the orthogonal matrix obtained from the SVD decomposition~\eqref{eq-SVD-P1} applied to \({\bm G}_{k\tau}\). It then follows that
\begin{align}
    & \Vert \bm{G}_t - \bm{P}_t\bm{P}_t^\top \bm{G}_t\Vert_F^2 \nonumber\\
    =& \Vert \bm{U}^\top(\bm{G}_t - \bm{P}_t\bm{P}_t^\top \bm{G}_t)\Vert_F^2 \nonumber\\
    =& \sum\limits_{j=1}^{m} \Vert \bm{u}_j^\top(\bm{G}_t - \bm{P}_t\bm{P}_t^\top \bm{G}_t)\Vert_F^2 \nonumber\\
    \overset{(a)}{=}& \sum\limits_{j \notin \mathcal{J}} \Vert \bm{u}_j^\top\bm{G}_t - \bm{0}^\top\Vert_F^2 + \sum\limits_{j \in \mathcal{J}} \Vert \bm{u}_j^\top\bm{G}_t - \bm{u}_j^\top\bm{G}_t\Vert_F^2 \nonumber\\
    =& \Vert \bm{G}_t \Vert_F^2 - \sum\limits_{j \in \mathcal{J}} \Vert\bm{u}_j^\top\bm{G}_t\Vert_F^2,\label{eq:sub_problem_derivation}
\end{align}
where \(\bm{u}_j\) is the \(j\)-th column of \({\bm U}\),  equality (a) follows from the fact that \(\bm{P}_t\) is constructed by selecting the columns indexed by \(\mathcal{J}\) from \(\bm{U}\) (see \eqref{eq:new_sub_problem}), and the last equality holds because \(\|\bm{G}_t\|_F^2 = \sum_{j=1}^m\|\bm{u}_j^\top\bm{G}_t\|_F^2\). From~\eqref{eq:new_sub_problem} and~\eqref{eq:sub_problem_derivation}, the index set \(\mathcal{J}\) is selected by choosing the top \(r\) components from the set \(\{\|\bm{u}_j^\top \bm{G}_t\|_F^2\}_{j=1}^m\). Since the sum of the top \(r\) components is at least an \(r/m\) fraction of the total, we have
\[
\sum_{j \in \mathcal{J}} \|\bm{u}_j^\top \bm{G}_t\|_F^2 \;\ge\; \frac{r}{m} \|\bm{G}_t\|_F^2.
\]
Substituting this bound into \eqref{eq:sub_problem_derivation} we prove \eqref{eq:contractive}. 
\end{proof}

\vspace{1mm}
\noindent 
\textbf{\(\mathsf{\bf Semi}\mbox{-}\mathsf{\bf Lazy}\ \mathsf{\bf SVD}\) implementation.}
Problem~\eqref{eq:new_sub_problem} can be solved efficiently, since~\eqref{eq:sub_problem_derivation} implies that the optimal projection matrix \(\bm{P} \in \mathcal{S}(\bm{U}, r)\) can be obtained by selecting the index set \(\mathcal{J}\) corresponding to the \(r\) largest values of \(\Vert \bm{u}_j^\top \bm{G}_t \Vert_F^2\), 
i.e., 
\begin{align}
\label{eq-semi-lazy-svd-2}
\mathcal{J} = \mathsf{arg\, top}_r\left( \left\{ \|\bm{u}_j^\top \bm{G}_t\|_F^2 \right\}_{j=1}^m \right),
\quad \bm{P}_t = [\bm{U}]_{:,\mathcal{J}},
\end{align}
where $\mathsf{arg\, top}_r(\cdot)$ returns $r$ indices. This approach avoids the need to solve the costly problem \eqref{eq:new_sub_problem}. In practice, \textbf{\(\mathsf{\bf Semi}\mbox{-}\mathsf{\bf Lazy}\ \mathsf{\bf SVD}\)} is implemented as follows
\begin{align}
\label{2enadq2e}
\hspace{-2mm}
\left\{
\begin{array}{ll}
\hspace{-1mm}\mbox{achieve $\bm U$ and $\bm P_t$ through \eqref{eq-SVD-P1}} ,& \mbox{if $\mathsf{mod}(t,\tau)=0$},  \\
\hspace{-1mm}\mbox{achieve $\bm P_t$ through \eqref{eq-semi-lazy-svd-2}}     ,& \mbox{otherwise}.
\end{array}
\right.
\end{align}
\textbf{\(\mathsf{\bf Semi}\mbox{-}\mathsf{\bf Lazy}\ \mathsf{\bf SVD}\)} \eqref{2enadq2e} performs a single SVD per period but  maintains the contractiveness of the compression operator.

\begin{algorithm}[tb]

    \caption{\(\mathsf{\bf Approx}\mbox{-}\mathsf{\bf  Top}\mbox{-}{\bf r}(\{{\bm G}_t^{(i)}\}_{i=1}^N, \bm{U}, r)\).}
    \label{alg:vec_top}

    \vspace{1pt}
    \JustifyAlgo{\hspace{-3.35mm}\textbf{Input}: $N$ nodes, rank $r$, global orthogonal matrix $\bm{U} \in \mathbb{R}^{m\times m}$, local gradient $\bm{G}_t^{(i)}  \in \mathbb{R}^{m\times n}$ for $i\in [N]$} \\
    {\bfseries Output}: Global rank-$r$ projector $\bm{P}_t \in \mathbb{R}^{m\times r}$. \\
    \textbf{(On $i$-th node)} \\
    Generate global random vectors $\{\bm v_j\}^m_{j=1}$ as in \eqref{eq:random-v}. \\
    \JustifyAlgo{\hspace{-4.75mm} Compute $\bm{\Lambda}^{(i)} \gets [\bm{\lambda}_1^{(i)}, \ldots, \bm{\lambda}_m^{(i)}]$ with $\bm{\lambda}_j^{(i)} = \bm{u}_{j}^\top\bm{G}_t^{(i)}\bm{v}_j $ for $j \in [m]$, where  $\bm{u}_{j}$ is $j$-th column of $\bm{U}$.} \\
    \vspace{0.5mm}
    Compute $\overline{\bm{\Lambda}} \gets\frac{1}{N}\sum_{i=1}^{N}{\bm{\Lambda}}^{(i)}$. \hfill (All-Reduce) \\
    \vspace{0.5mm}
    \JustifyAlgo{\hspace{-4.8mm} Compute $\overline{\bm{\Sigma}} \gets (\overline{\bm{\Lambda}})^2$ with element-wise square operation.} \\
    Compute $\mathcal{J} = \mathsf{arg\, top}_r\left( \overline{\bm{\Sigma}}\right)$ and $\bm P_t \gets \bm{U}_{:,\mathcal{J}}$. \\
    \Return $\bm{P}_t$.
\end{algorithm}
\vspace{-5mm}

\subsection{Approximate global Top-$r$ projection}
When computing \(\bm{P}_t\) via~\eqref{eq-semi-lazy-svd-2}, we must aggregate the globally averaged gradient \({\bm{G}}_t = \frac{1}{N}\sum_{i=1}^{N} \bm{G}_t^{(i)}\), where \(\bm{G}_t^{(i)}\) denotes the local gradient on the \(i\)-th node at iteration \(t\). This requires full communication of all local gradients \(\bm{G}_t^{(i)}\), which can be prohibitively expensive in large-scale distributed systems.  Worse still, since \({\bm{G}}_t\) varies at every iteration, a naive implementation of~\eqref{eq-semi-lazy-svd-2} would require transmitting \(\bm{G}_t^{(i)}\) at each iteration—undermining the communication efficiency that low-rank compression is intended to provide.

To mitigate this issue, we propose a strategy that approximately identifies the top \(r\) indices in~\eqref{eq-semi-lazy-svd-2} without requiring transmission of the full local gradients \(\bm{G}_t^{(i)}\). The key idea is to generate \(m\) independent projection vectors \(\bm{v}_1, \ldots, \bm{v}_m\), each sampled from the standard normal distribution in \(\mathbb{R}^n\), i.e.,
\begin{align}
\label{eq:random-v}
\bm{v}_j \sim \mathcal{N}(\bm{0}, \bm{I}_n), \quad j = 1, \ldots, m,
\end{align}
using a synchronized random seed across all nodes. During the implementation of~\eqref{eq-semi-lazy-svd-2}, rather than transmitting \(\bm{G}_t^{(i)}\) or the full set \(\{\bm{u}_j^\top \bm{G}_t^{(i)}\}_{j=1}^m\) during the global All-Reduce operation, each node $i$ transmits only $m$ single scalars \(\bm{\lambda}_j^{(i)} = \bm{u}_j^\top \bm{G}_t^{(i)} \bm{v}_j\) for $j\in [m]$. The global average is then computed as
\begin{align}
\label{eq:lambda-bar}
\overline{\bm{\lambda}}_j = \frac{1}{N} \sum_{i=1}^N \bm{\lambda}_j^{(i)} \ \mbox{with} \ \bm{\lambda}_j^{(i)} = \bm{u}_j^\top \bm{G}_t^{(i)} \bm{v}_j, \quad \forall j \in [m].
\end{align}
The following proposition shows that $
\overline{\bm{\lambda}}_j^2 = \left(\frac{1}{N} \sum_{i=1}^{N} \bm{\lambda}_j^{(i)}\right)^2$ 
is an unbiased estimator of $
\|\bm{u}_j^\top {\bm{G}}_t\|_F^2$ in \eqref{eq-semi-lazy-svd-2}. 

\begin{center}
\setlength\fboxsep{4pt}
\colorbox{gray!20}{\parbox{\dimexpr\linewidth-2\fboxsep\relax}{
\begin{proposition}\label{prop:unbiased-lambda}
Let $\overline{\bm{\lambda}}_j = \frac{1}{N} \sum_{i=1}^N \bm{\lambda}_j^{(i)}$ with each \(\bm{\lambda}_j^{(i)} = \bm{u}_j^\top \bm{G}_t^{(i)} \bm{v}_j\) and $\bm{v}_j \sim \mathcal{N}(\bm{0}, \bm{I}_n)$. It holds that 
\begin{align}
\mathbb{E}[\overline{\bm{\lambda}}_j^2]  = \|\bm{u}_j^\top {\bm{G}}_t\|_F^2.
\end{align}
\end{proposition}}}
\end{center}

\begin{algorithm}[!tb]
\caption{GreedyLore algorithm}
\label{alg:greedylore}

\vspace{1pt}
\JustifyAlgo{\hspace{-4.9mm} {\bfseries Input}: $N$ nodes, number of total iterations $T$, subspace changing frequency $\tau$, rank $r$, $\bm{X}_{0} \in \mathbb{R}^{m \times n}$ and  $\bm{E}_{-1}^{(i)} = \bm{0}$ for node $i \in [N]$ with shape $m \le n$. Initialize $\bm U = \bm I_m$.} \\
\vspace{2pt}
\textbf{Output}: Sequence of model weights $\{\bm{X}_t\}_{t=0}^{T+1}$.

\For{$t=0,\ldots,T$}{
    \textbf{(On $i$-th node)} \\
    \colorbox{green!30}{$\bm{G}_t^{(i)} \gets \nabla F_i({\bm{X}}_t; {\bm{\xi}}_t^{(i)}) + \bm{E}_{t-1}^{(i)}$ with local data ${\bm{\xi}}_t^{(i)}$.} \\
    \colorbox{purple!15}{$\bm{P}_t, \bm{G}_t, \bm{U} \gets \mathsf{\bf Semi}\mbox{-}\mathsf{\bf Lazy}\mbox{-}\mathsf{\bf SVD}(\{\bm{G}_t^{(i)}\}_{i=1}^N, \bm{U}, t)$.} \\
    $\bm{R}_t^{(i)} \gets {\bm{P}_t}^\top{\bm{G}}_t^{(i)}$. \\
    \colorbox{green!30}{$\bm{E}_{t}^{(i)} \gets \bm{G}_t^{(i)} \hspace{-1mm}- \hspace{-1mm}\bm{P}_t\bm{R}_t^\top \ \mathsf{\bf if\ } \mathsf{mod}(t, \tau) \ne 0\ \mathsf{\bf otherwise}\  \bm{0}$.} \\
    ${\bm{R}}_t \gets \frac{1}{N}\sum_{i=1}^{N}{\bm{R}}_t^{(i)}$. \hfill (All-Reduce)   \\
    $ \hat{\bm{G}}_t \gets \bm{P}_t\bm{R}_t\ \mathsf{\bf if}\  \mathsf{mod}(t, \tau) \ne 0$ \textbf{otherwise} $\bm{G}_t$. \\
    $\bm{X}_{t+1} \gets \mathsf{\bf Optimizer}(\bm{X}_{t}, \hat{\bm{G}}_t, \gamma)$. \\
}
\Return $\{\bm{X}_{t}\}_{t=0}^{T+1}$.

\vspace{2mm}
\sub{$\mathsf{\bf Semi}\mbox{-}\mathsf{\bf Lazy}\mbox{-}\mathsf{\bf SVD}(\{\bm{G}_t^{(i)}\}_{i=1}^N,\bm{U}, t)$}{
    \uIf{$\mathsf{mod}(t, \tau) = 0$}{
        $\bm{G}_t \gets \frac{1}{N}\sum_{i=1}^{N}{\bm{G}_t}^{(i)} $.\hfill (All-Reduce) \\
        $\bm{U},\bm{\Sigma}, \bm{V} \gets \mathrm{\bf SVD}(\bm{G}_t),\ \bm P_t \gets \bm{U}_{:,:r}$. \\
        \vspace{-0.5mm}
        \Return $\bm{P}_t,  \bm{G}_t, \bm{U}$.
    }\Else{
        $\bm{P}_{t} \gets \mathsf{\bf Approx} \mbox{-} \mathsf{\bf Top} \mbox{-}{\bf r} (\{{\bm G}_t^{(i)}\}_{i=1}^N, \bm{U}, r)$. \\
        \Return $\bm{P}_{t}, \bm{0}, \bm{U}$. \\
    }
}

\end{algorithm}

\begin{proof}
Since all nodes share the same random projection vectors $\bm{v}_j$, it is straightforward to verify that:
\begin{align}
    \mathbb{E}[\overline{\bm{\lambda}}_j^2] 
    &= \mathbb{E}\left[\bm{u}_j^\top {\bm{G}}_t\, \bm{v}_j \,\bm{v}_j^\top \,\left({\bm{G}}_t\right)^\top \bm{u}_j\right] \nonumber\\
    &= \bm{u}_j^\top {\bm{G}}_t\, \mathbb{E}\left[\bm{v}_j \bm{v}_j^\top\right]\, \left({\bm{G}}_t\right)^\top \bm{u}_j \nonumber\\
    &= \bm{u}_j^\top {\bm{G}}_t\, \bm{I}_n \,\left({\bm{G}}_t\right)^\top \bm{u}_j = \|\bm{u}_j^\top {\bm{G}}_t\|_F^2.
\end{align}
The first equality holds since $\overline{\bm{\lambda}}_j=\bm{u}_j^\top (\frac{1}{N}\sum_{i=1}^N{\bm G}_t^{(i)}) {\bm v}_j$.
\end{proof}
With Proposition~\ref{prop:unbiased-lambda}, we propose the \(\mathsf{\bf Semi}\mbox{-}\mathsf{\bf Lazy}\ \mathsf{\bf SVD}\) operator with approximate global Top-$r$ projection: 
\begin{align}
\label{eq-semi-lazy-svd-3}
\mathcal{J} = \mathsf{arg\, top}_r\left( \{ \overline{\bm{\lambda}}_j^2 \}_{j=1}^m \right),
\quad \bm{P}_t = [\bm{U}]_{:,\mathcal{J}}.
\end{align}
Furthermore, owing to the independence of the projection vectors $\{\bm{v}_j\}_{j=1}^m$, we can demonstrate that the compressor $\mathcal{C}_t(\bm{G}_t) = \bm{P}_t \bm{P}_t^\top \bm{G}_t$ maintains its contractive property when $\bm{P}_t$ is generated according to~\eqref{eq-semi-lazy-svd-2}, see Appendix~\ref{app:proof}. The implementation of the approximate global \(\mathsf{Top}\text{-}r\) projection is detailed in Algorithm~\ref{alg:vec_top}, where only \({\bm \Lambda}^{(i)} \in \mathbb{R}^m\) is transmitted per iteration, resulting in significant communication savings. 

\subsection{GreedyLore algorithm}
With the aforementioned techniques, we propose a novel \underline{\textbf{Greedy}} \underline{\textbf{Lo}}w-\underline{\textbf{R}}ank Gradi\underline{\textbf{e}}nt Compression algorithm, termed \textbf{GreedyLore}, for distributed learning (see Algorithm~\ref{alg:greedylore}). Compared to the preliminary framework in Algorithm~\ref{alg:prelimnary}, GreedyLore introduces two key improvements. First, it incorporates error feedback to correct the bias introduced by greedy communication compression (see the step highlighted in \colorbox{green!30}{green}). Second, it employs the \(\mathsf{\bf Semi}\mbox{-}\mathsf{\bf Lazy}\ \mathsf{\bf SVD}\), which adapts \(\bm{P}_t\) to the time-varying \(\bm{G}_t\), thereby ensuring that the compressor remains contractive throughout all iterations and preserving the effectiveness of error feedback (see the steps highlighted in \colorbox{purple!15}{red}). These two improvements form the foundation for establishing the algorithm’s convergence guarantees. Similar to Algorithm~\ref{alg:prelimnary}, since each node obtains the globally averaged gradient \(\bm{G}_t\) when \(\mathsf{mod}(t, \tau) = 0\) in \(\mathsf{\bf Semi}\mbox{-}\mathsf{\bf Lazy}\ \mathsf{\bf SVD}\), GreedyLore directly use it in the optimizer instead of \(\hat{\bm{G}}_t\). In addition, we reset \(\bm{E}_t^{(i)} = \bm{0}\) when \(\mathsf{mod}(t, \tau) = 0\), as no communication compression is applied at these iterations.

\subsection{Analysis of Communication Efficiency}

This subsection analyzes and compares the communication efficiency of various low-rank algorithms. The results are summarized in Table~\ref{table:comm_cmp}. {ATOMO}~\cite{wang2018atomo} performs SVD on local gradients, resulting in basis vectors that may differ across nodes. Consequently, the all-reduce operation cannot be directly applied. In contrast, {PowerSGD}~\cite{vogels2019powersgd} maintains a consistent projection matrix across nodes, and {LoRA}~\cite{hu2022lora} directly learns a low-rank structure of the parameters. These methods therefore support all-reduce operations in data-parallel settings. However, they require communicating two matrices from the low-rank decomposition at each iteration, incurring a communication cost of \(mr + nr\) per iteration.

For algorithms employing lazy SVD, such as {GaLore}~\cite{zhao2024galore} and {GreedyLore}, SVD is performed only once every \(\tau\) iterations. In the iteration where SVD is executed, the global gradient $\bm G_t$ is required, leading to communication of a full \(m \times n\) matrix. At every iteration, {GaLore} communicates only a projected low-dimensional matrix of size \(r \times n\). The total communication over \(\tau\) iterations is \(\tau nr + mn\), corresponding to an average per-iteration communication cost of \(nr + \frac{mn}{\tau}\). In the {GreedyLore} algorithm, in addition to transmitting the \(r \times n\) matrix, a vector of length \(m\) is also communicated in Algorithm~\ref{alg:vec_top}, increasing the average per-iteration communication by $m$ compared to {GaLore}. However, since the dimensions \(m\) and \(n\) are of the same order in most neural network layers. When \(r\) is chosen moderately large such that \(nr \gg m\), the additional communication overhead becomes negligible.

\section{Convergence Analysis}\label{sec:convergence}
Throughout this section, we let $\bm{G}_t^{(i)} = \nabla F_i(\bm{X}_t;\bm{\xi}_t^{(i)})+\bm{E}_{t-1}^{(i)}$, $\bm{G}_t = \frac{1}{N}\sum_{i=1}^N \bm{G}_t^{(i)}$ and $\hat{\bm G}_t = \bm P_t \bm P_t^\top \bm G_t$ if $\mathsf{mod}(t,\tau)\ne0$ otherwise $\bm{G}_t$. Assumptions~\ref{asp:LL} and~\ref{asp:stochastic} are standard in the convergence analysis of stochastic optimization.

\begin{assumption}\label{asp:LL}
    $f(\bm{X})$ is $L$-smooth and lower bounded, \textit{i.e.},
    \begin{align*}
        \|\nabla f(\bm{X})-\nabla f(\bm{Y})\|_F\le L\|\bm{X}-\bm{Y}\|_F,\ \forall\ \bm{X},\bm{Y}\in\mathbb{R}^{m\times n},
    \end{align*}
    and $\inf_{\bm{X}\in\mathbb{R}^{m\times n}}f(\bm{X})>-\infty$. 
\end{assumption}
\begin{assumption}\label{asp:stochastic}
    Stochastic gradient oracle $\nabla F_i(\bm{X};{\bm{\xi}})$ satisfies
    \begin{align*}
        &\mathbb{E}_{\bm{\xi}\sim\mathcal{D}_i}[\nabla F_i(\bm{X};{\bm{\xi}})]=\nabla f_i(\bm{X}),\\
        &\mathbb{E}_{\bm{\xi}\sim\mathcal{D}_i}[\|\nabla F_i(\bm{X};{\bm{\xi}})-\nabla f_i(\bm{X})\|_2^2]\le\sigma^2.
    \end{align*}
\end{assumption}
Assumption~\ref{asp:BG} ensures that the full gradient is uniformly bounded and will be used in the convergence analysis of the MSGD optimizer. Assumption~\ref{asp:BSG} bounds the stochastic gradients and is commonly adopted in the analysis of Adam-like algorithms. Assumption~\ref{asp:clcu} constrains the adaptive learning rates within fixed bounds by requiring the inverse square root of the second-order momentum term to lie between two constants; this condition can be easily satisfied via clipping.
\begin{assumption}\label{asp:BG}The gradient is uniformly upper bounded:
    \begin{align}
    \|\nabla f(\bm{X})\|_F\le B,\quad\forall\ \bm{X}\in\mathbb{R}^{m\times n}.
    \end{align}
\end{assumption}
\begin{assumption}\label{asp:BSG}
    Stochastic gradient oracle satisfies 
    \begin{align}
        \|\nabla F_i(\bm{X};{\bm{\xi}})\|_F\le B_s,\quad\forall\ \bm{X}\in\mathbb{R}^{m\times n},\forall\ {\bm{\xi}}.
    \end{align}
\end{assumption}
\begin{assumption}\label{asp:clcu}
    There exist constants $0<c_l<c_u$ such that $c_l\le\left[\frac{1}{\sqrt{\tilde{\bm{V}}_{t}+\epsilon}}\right]_{ij}\le c_u$ holds for any $1\le i\le m$, $1\le j\le n$.

\end{assumption}

\begin{table}[t]  
\renewcommand{\arraystretch}{1.5}
\begin{center}
\caption{ \small  Communication comparison of existing low-rank algorithms for compressing an \(m \times n\) matrix. In {GaLore} and {GreedyLore}, SVD is performed once every \(\tau\) iterations. ``All-reduce'' indicates whether the algorithm supports efficient all-reduce operations.}
\tabcolsep=2pt
\begin{tabular}{lP{1.5cm}P{2cm}P{2.2cm}}
    \toprule
    \textbf{Algorithm} &
    \textbf{All-Reduce} &
    \textbf{\renewcommand{\arraystretch}{1.0}\begin{tabular}[c]{@{}c@{}}Error-Feedback\end{tabular}} &
    \textbf{\renewcommand{\arraystretch}{1.0}\begin{tabular}[c]{@{}c@{}}Communication \\ per Iteration\end{tabular}} \\
    \midrule
    ATOMO\cite{wang2018atomo} & \xmark & \cmark & $mr + nr$ \\
    PowerSGD\cite{vogels2019powersgd} & \cmark & \cmark & $mr + nr$ \\
    LoRA\cite{hu2022lora}     & \cmark & \xmark & $mr + nr$ \\

    GaLore\cite{zhao2024galore}   & \cmark & \xmark & $nr + \frac{mn}{\tau}$ \\
    \rowcolor{purple!15} GreedyLore & \cmark & \cmark & $nr + m + \frac{mn}{\tau}$  \\
    \bottomrule
\end{tabular}
\label{table:comm_cmp}
\end{center}
\vspace{-4mm}
\end{table}

\noindent We now present the convergence results of {GreedyLore}.
\begin{center}
\setlength\fboxsep{4pt}
\colorbox{gray!20}{\parbox{\dimexpr\linewidth-2\fboxsep\relax}{
\begin{theorem}[GreedyLore Convergence with MSGD]\label{thm:cov-msgd} Under Assumptions \ref{asp:LL}, \ref{asp:stochastic} and \ref{asp:BG}, if we let $\gamma\le 1/(4L)$, GreedyLore with Momentum SGD converges as 
\begin{align}
    &\frac{1}{T+1}\sum_{t=0}^T\mathbb{E}[\|\nabla f(\bm{X}_t)\|_F^2]\le\nonumber\\
    &\quad\quad\frac{2\Delta_0}{\gamma(T+1)}+\frac{40\gamma^2L^2(B^2+\sigma^2)}{(1-\beta_1)^2\delta}+\frac{2\gamma L\sigma^2}{N},\label{eq:thm-msgd}
\end{align}
where $\delta:=r/m$, $\Delta_0:=f(\bm{X}_0)-\inf_{\bm{X}}f(\bm{X})$. If we further choose $\gamma=\left(4L+\sqrt{\frac{L\sigma^2(T+1)}{n\Delta_0}}+\sqrt[3]{\frac{L^2(B^2+\sigma^2)(T+1)}{(1-\beta_1)^2\delta\Delta_0}}\right)^{-1}$, GreedyLore with Momentum SGD converges as follows
\begin{align}
    &\frac{1}{T+1}\sum_{t=0}^T\mathbb{E}[\|\nabla f(\bm{X}_t)\|_F^2]= \\
    &\quad\mathcal{O}\left(\sqrt{\frac{L\Delta_0\sigma^2}{N(T+1)}}+\sqrt[3]{\frac{L^2\Delta_0^2(B^2+\sigma^2)}{\delta(T+1)^2}}+\frac{L\Delta_0}{T+1}\right). \nonumber
\end{align}
\end{theorem}}}
\end{center}

\noindent
\begin{center}
\setlength\fboxsep{4pt}\colorbox{gray!20}{\parbox{\dimexpr\linewidth-2\fboxsep\relax}{\begin{theorem}[GreedyLore Convergence with Adam]\label{thm:converge-adam}
Under Assumptions \ref{asp:LL}, \ref{asp:stochastic}, \ref{asp:BSG} and \ref{asp:clcu}, if we let $\gamma\le c_l/(4Lc_u^2)$, GreedyLore with Adam converges as 
\begin{align*}
    &\frac{1}{T+1}\sum_{t=0}^T\mathbb{E}[\|\nabla f(\bm{X}_t)\|_F^2]\le\frac{2\Delta_0}{\gamma c_l(T+1)}+\frac{40\gamma^2L^2c_u^4B_s^2}{(1-\beta_1)^2c_l^2\delta}\nonumber\\
    &+\frac{4c_u(2+\gamma Lc_u)\tau B_s^2}{c_l(1-\beta_1)(T+1)}+\frac{4\gamma L\tau^2c_u^2B_s^2}{(1-\beta_1)^2c_l(T+1)}+\frac{2\gamma Lc_u^2\sigma^2}{Nc_l},
\end{align*}
where $\delta:=r/m$, $\Delta_0:=f(\bm{X}_0)-\inf_{\bm{X}}f(\bm{X})$. If we further choose $\gamma$ properly,
GreedyLore with Adam converges as
\begin{align}
    &\frac{1}{T+1}\hspace{-0.5mm}\sum_{t=0}^T\mathbb{E}[\|\nabla f(\bm{X}_t)\|_F^2]= \\
&\mathcal{O}\hspace{-0.6mm}\left(\hspace{-0.6mm}\sqrt{\frac{L\Delta_0\sigma^2}{N(T+1)}}\hspace{-0.6mm}+\hspace{-0.6mm}\sqrt[3]{\frac{L^2\Delta_0^2B_s^2}{\delta(T+1)^2}}\hspace{-0.6mm}+\hspace{-0.6mm}\frac{\tau B_s(B_s\hspace{-0.6mm}+\hspace{-0.6mm}\sqrt{\Delta_0})\hspace{-0.6mm}+\hspace{-0.6mm}L\Delta_0}{T+1}\hspace{-0.6mm}\right).\hspace{-1.5mm} \nonumber 
\end{align}
\end{theorem}}}
\end{center}
\begin{remark}
Theorems~\ref{thm:cov-msgd} and~\ref{thm:converge-adam} show that {GreedyLore}, when combined with either the MSGD or Adam optimizer, achieves a convergence rate of \(\mathcal{O}\left(\sqrt{L\Delta_0\sigma^2/(NT)}\right)\) as \(T \to \infty\), matching that of vanilla parallel SGD with full-dimensional communication. This demonstrates that the use of low-rank communication compression does not compromise the algorithm’s asymptotic convergence rate. In contrast, existing low-rank compression methods~\cite{robert2024ldadam,zhao2025separate} fail to attain this rate (see Table~\ref{table:alg_cmp}), indicating degraded convergence in theory.
\end{remark}

\begin{remark}
Since {GreedyLore} achieves a convergence rate of \(\mathcal{O}(1/\sqrt{NT})\) as \(T \to \infty\), it requires \(T = \mathcal{O}(1/(N\epsilon^2))\) iterations to reach a target accuracy \(\epsilon\), demonstrating that the required iteration count decreases linearly with the number of nodes \(N\). Therefore, {GreedyLore} achieves linear speedup in terms of iteration complexity. To the best of our knowledge, this is the first linear speedup result for distributed algorithms using low-rank communication compression (see Table~\ref{table:alg_cmp}).
\end{remark}

\begin{remark}
LDAdam~\cite{robert2024ldadam} relies critically on the strong contractive compressor assumption (Definition~\ref{compressor-assumption}), which enforces \(\mathbb{E}_{\mathcal{C}}\bigl\|\mathcal{C}(\bm{X})-\bm{X}\bigr\|_F^2 \le (1-\delta)\|\bm{X}\|_F^2\) with a fixed \(\delta>0\) for all \(X\). In contrast, by employing a dynamic update of the projection matrix \(\bm{P}_t\), GreedyLore inherently satisfies the contractive property without imposing any compressor‐specific assumptions. Details are provided in the proof of Lemma~\ref{lm:contractive}.
\end{remark}

\section{Experimental Evaluation}\label{sec:exp}
In this section, we evaluate GreedyLore when pre-training and fine-tuning of deep neural networks on tasks in computer vision and natural language processing. We also analyze memory consumption and wall‐clock time to compare GreedyLore against other communication‐efficient algorithms.

\begin{table*}[hb!]
\vspace{-2mm}
\renewcommand{\arraystretch}{1.5}
\begin{center}
\caption{\small Evaluation results of fine-tuning RoBERTa model on GLUE benchmark. The highest score in each group is bolded.}
\begin{tabular}{llccccccccc}
\toprule
Algorithm & \textbf{SST-2}   & \textbf{CoLA}   & \textbf{MRPC} & \textbf{STS-B} & \textbf{RTE}    & \textbf{QNLI}   & \textbf{QQP} & \textbf{MNLI}  & \textbf{Avg} \\ 
\midrule
AdamW                   & 0.9457 & 0.6224 & 0.9261   & 0.9109        & 0.7942 & 0.9240 & 0.9172 & 0.8724 &0.8641  \\
\midrule
1-bit Adam & 0.9106 & 0.6082 & 0.9303 & 0.9095 & 0.7365 & 0.9207 & 0.9050 & 0.8277 & 0.8436 \\
top-2\%    & 0.9427 & 0.5982 & 0.9307 & 0.9085 & 0.7870 & 0.9277 & 0.9179 & 0.8694 & 0.8603 \\
rand-5\%   & 0.9404 & 0.5806 & 0.9277 & 0.9068 & 0.7834 & 0.9220 & 0.9162 & 0.8720 & 0.8561 \\
\midrule
PowerSGD(rank=8)          & 0.9415 & 0.6115 & 0.9236   & 0.9105        & 0.7870 & \textbf{0.9268} & \textbf{0.9186} & \textbf{0.8727} & 0.8615  \\
GaLore(rank=8) & 0.9438  & 0.5956  & 0.9196 & 0.9094 & 0.7942  & 0.9198  & 0.9163  & 0.8669 & 0.8582 \\
SEPARATE(rank=8) & 0.9450 & \textbf{0.6232} & 0.9244   & \textbf{0.9121}       & 0.7653 & 0.9255 & 0.9173 & 0.8705 & 0.8604  \\
RSO(rank=8) & 0.9415  & 0.5956  & 0.9196 & 0.9094 & 0.7942  & 0.9198  & 0.9163  & 0.8669 & 0.8579 \\
\textbf{GreedyLore}(rank=8)  & \textbf{0.9450} & 0.6107 & \textbf{0.9268}   & \textbf{0.9121}        & \textbf{0.8014} & 0.9239 & 0.9175 & 0.8701 & \textbf{0.8634}  \\
\midrule
PowerSGD(rank=16)          & 0.9392 & 0.6157 & 0.9277   & 0.9102        & 0.7870 & \textbf{0.9251} & \textbf{0.9180} & \textbf{0.8721} & 0.8619  \\
GaLore(rank=16)           & 0.9404 & \textbf{0.6358} & 0.9261   & 0.9074        & 0.7798 & 0.9240 & 0.9175 & 0.8682 & 0.8624  \\
SEPARATE(rank=16) & 0.9404 & 0.6135 & 0.9201   & 0.9104        & 0.7726 & 0.9240 & 0.9172 & 0.8713 & 0.8587  \\
RSO(rank=16)  & \textbf{0.9472} & 0.6191 & 0.9171   & 0.8987        & 0.7726 & 0.9160 & 0.9035 & 0.8556 & 0.8537  \\
\textbf{GreedyLore}(rank=16)  & 0.9461 & 0.6257 & \textbf{0.9291}   & \textbf{0.9116} & \textbf{0.7906} & 0.9218 & 0.9171 & 0.8619 & \textbf{0.8640}  \\ 
\bottomrule
\end{tabular}
\label{table:glue}
\end{center}
\vspace{-3mm}
\end{table*}

\subsection{Pre‐training with GreedyLore}
We assess GreedyLore in the context of pre‐training deep models. We compare it to the following baselines: AdamW~\cite{kingma2014adam, loshchilov2017decoupled}, 1‐bit Adam~\cite{tang20211}, PowerSGD~\cite{vogels2019powersgd}, and GaLore~\cite{zhao2024galore}, following the experimental protocol in~\cite{zhao2025separate}. Consistent with GaLore, we apply gradient compression exclusively to the two‐dimensional parameters within transformer layers, while retaining the full gradients for embedding and output layers. After a prescribed number of warm‐up iterations, we start compression using the same schedule as PowerSGD~\cite{vogels2019powersgd}. For convolutional neural networks, we reshape four‐dimensional tensors into two‐dimensional before applying low‐rank compression, as per the methodology in PowerSGD.

\noindent \textbf{Pre‐training ResNet on CIFAR.}
We evaluate GreedyLore on a ResNet‐18~\cite{he2016deep} model pre‐trained on the CIFAR‐10 and CIFAR‐100 datasets~\cite{krizhevsky2009learning}. ResNet‐18 contains approximately 11 million parameters. CIFAR‐10 contains 60,000 color images of size $32\times32$ pixels across 10 classes (50,000 for training, 10,000 for testing), while CIFAR‐100 follows the same format with 100 classes. We train for 40 epochs using a global batch size of $4\times32$, and report test accuracies in Figure~\ref{fig:cifar}. Detailed hyperparameter settings are provided in Appendix~\ref{sec:hyper-param}.
As shown in Figure~\ref{fig:cifar}, GreedyLore outperforms low‐rank compression schemes (GaLore and PowerSGD) as well as quantization methods such as 1‐bit Adam. Moreover, GreedyLore’s accuracies closely match those of AdamW~\cite{loshchilov2017decoupled}, demonstrating its efficacy for vision tasks in deep learning.

\begin{figure}[h]
\vspace{-2mm}
\centering
\includegraphics[width=2.5in]{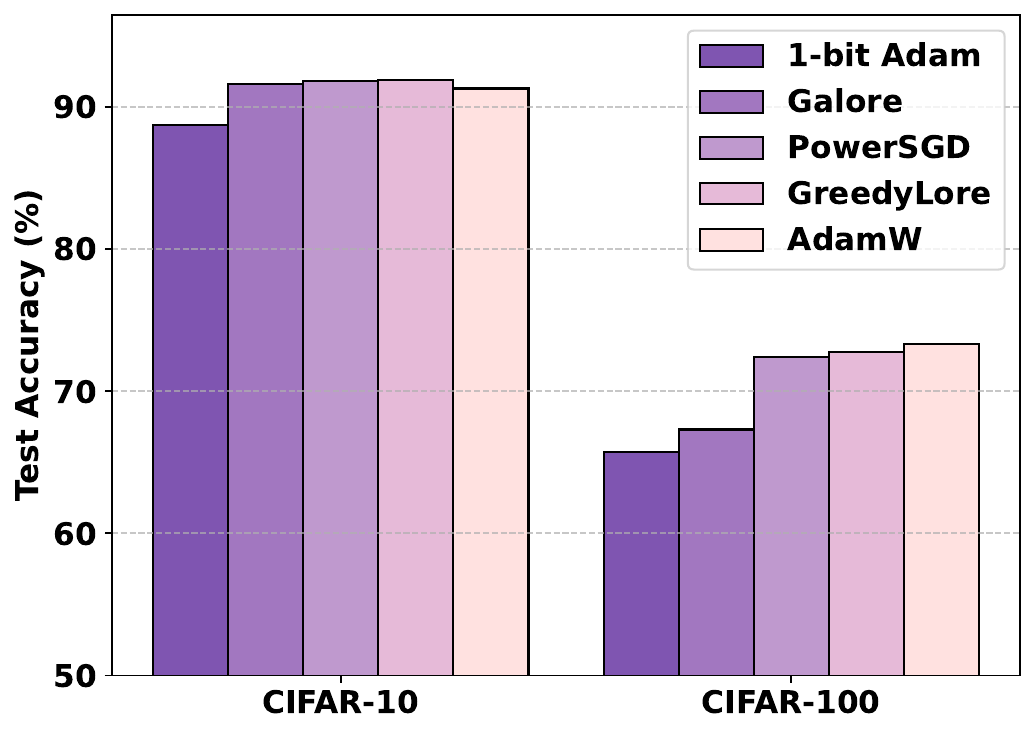}
\caption{Testing accuracy of pre-training ResNet-18 model on CIAFR-10 and CIFAR-100 dataset after training for 40 epochs.}
\label{fig:cifar}
\vspace{-3mm}
\end{figure}

\noindent \textbf{Pre‐training LLaMA on C4.}
We pre‐train autoregressive LLaMA transformer models of three scales (60 M, 130 M, and 350 M parameters) on the Colossal Clean Crawled Corpus (C4)~\cite{raffel2020exploring} with a data‐parallel degree of 4. LLaMA is an autoregressive transformer family introduced by Touvron et al.~\cite{touvron2023llama} that achieves state-of-the-art performance on diverse language tasks. C4 is a large-scale, cleaned web-crawl dataset optimized for robust language model training. Table~\ref{table:c4} presents the minimum validation perplexities attained by each optimizer. At a low compression rank (e.g., $r=8$), GreedyLore and PowerSGD match the performance of AdamW, while GaLore lags behind. Results for PowerSGD at higher ranks are omitted due to prohibitive runtimes. In contrast, GreedyLore markedly outperforms GaLore under these conditions.

Additionally, we extend the largest configuration to LLaMA models with 1B parameters. Due to resource constraints, we report the training loss over the first 100,000 optimization steps. Figure~\ref{fig:1b} shows that, at a learning rate of $5\times10^{-4}$, GreedyLore converges comparably to AdamW and consistently outperforms GaLore throughout training.

\begin{figure}[ht!]
\centering
\includegraphics[width=2.5in]{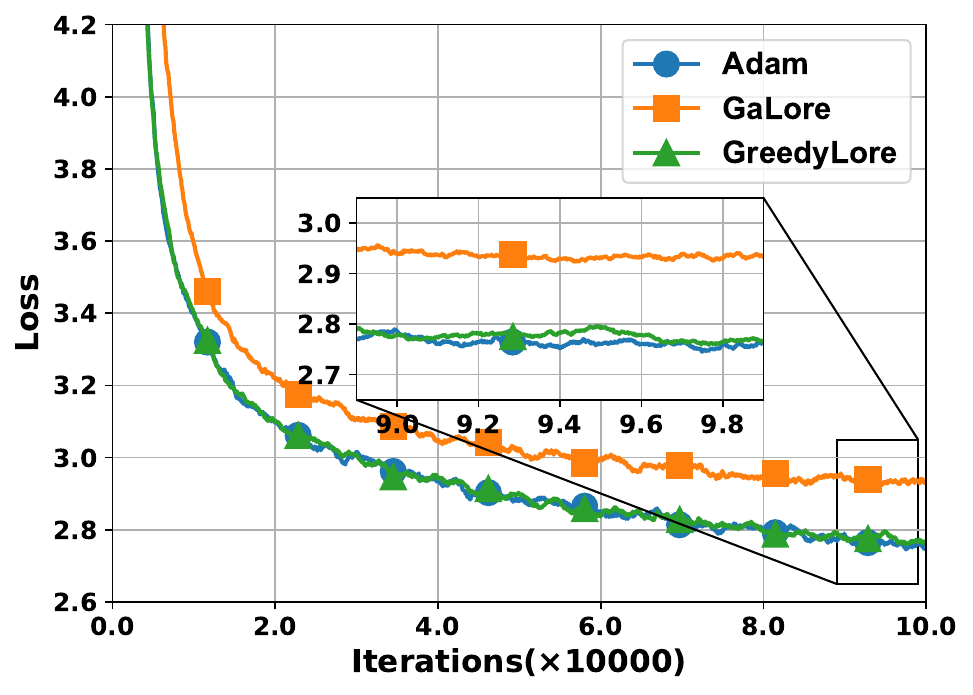}
\caption{Training loss of pre-training LLaMA-1B model on C4 dataset.}
\label{fig:1b}
\vspace{-4mm}
\end{figure}

\subsection{Fine‐tuning with GreedyLore}\label{sec:finetune}
We evaluate GreedyLore for fine‐tuning medium‐sized language models. We expand our baselines to include sparsification compressors—specifically Top-$k$ (with $k=2\%$) and Random-$k$ (with $k=5\%$) algorithms—to highlight the superior performance of low‐rank compression methods.

We fine‐tune the pretrained RoBERTa‐base~\cite{liu2019roberta} model on the General Language Understanding Evaluation (GLUE) benchmark~\cite{wang2018glue} using a cluster of four NVIDIA RTX 4090 GPUs with memory of 24\,GB each. RoBERTa‐base is a 110M‐parameter transformer that extends BERT via dynamic masking. GLUE comprises diverse natural language understanding tasks designed to assess general‐purpose capabilities. We fine‐tune for 10 epochs per task and report test metrics in Table~\ref{table:glue} following the convention in \cite{zhao2024galore}.

As shown in Table~\ref{table:glue}, GreedyLore preserves fine‐tuning performance under aggressive low‐rank compression, outperforming GaLore and 1-bit Adam while closely matching the full‐precision AdamW baseline~\cite{loshchilov2017decoupled}. As the compression rank increases, GreedyLore  recovers accuracy of AdamW and surpasses other compression algorithms. These results confirm that GreedyLore delivers communication‐efficient optimization without sacrificing downstream performance.

\begin{table}[H]  
\renewcommand{\arraystretch}{1.5}
\begin{center}
\caption{ \small Validation perplexity of pre-training LLaMA models on C4 dataset. Training settings are shown in the last two rows.}
{\tabcolsep=2pt
\begin{tabular}{lcccccc}
\toprule
Algorithm & \multicolumn{2}{c}{\textbf{LLaMA-60M}} & \multicolumn{2}{c}{\textbf{LLaMA-130M}} &\multicolumn{2}{c}{\textbf{LLaMA-350M}}\\
\midrule
AdamW      &  \multicolumn{2}{c}{33.44}  &  \multicolumn{2}{c}{24.73} & \multicolumn{2}{c}{18.66} \\
\midrule
SEAPTATE   & 45.31  & 35.23 & 26.31  & 25.88 & 21.35   & 20.12 \\
RSO   & 44.54  & 35.43 & 36.15  & 25.85 & 27.39       &   19.66  \\
\midrule
PowerSGD   & 34.93  & -       & 26.31  & -       & 21.35   & -        \\
GaLore     & 53.93  & 34.88   & 41.21      & 25.36   & 33.30       & 18.95    \\
GreedyLore & \textbf{34.75}  & \textbf{31.67}   & \textbf{25.65}  & \textbf{24.26}   & \textbf{19.46}   & \textbf{18.62}    \\
\midrule
$r/d_\mathrm{model}$       & 32/256 & 128/256 & 32/512 & 256/768 & 32/1024 & 256/1024 \\
Training Tokens       & 1.1B & 1.1B & 2.2B & 2.2B & 6.4B & 6.4B \\
\bottomrule
\end{tabular}
}
\vspace{-4mm}
\label{table:c4}
\end{center}
\end{table}

\subsection{Wall‐Clock Performance Analysis}\label{sec:wallclock}

Table~\ref{table:c4_time} reports the average per‐iteration wall‐clock time (in seconds) for pre‐training various LLaMA model sizes on the C4 dataset. For each model, we instrumented 500 consecutive iterations and computed the mean elapsed time. All experiments were conducted on a machine equipped with four NVIDIA RTX 4090 GPUs using the NCCL communication backend with shared‐memory (SHM) transfers. We fixed the compression rank to $r=32$ and the per‐GPU batch size to 128 (except for the 1B-parameter model, where we reduced the batch size to 64 due to memory constraints).

When the model is small (e.g., LLaMA-60M), any time savings from communication compression are largely negated by the additional synchronization and computational overhead. As model size increases, communication overhead becomes the dominant bottleneck. As a result, gradient compression yields progressively larger time savings. In particular, for the LLaMA-1B model, GreedyLore reduces per‐iteration wall‐clock time by approximately 27\% relative to standard AdamW, consistent with the results shown in ~\cite{zhang2023evaluation}.

Furthermore, for compression ranks beyond a moderate threshold (e.g., $r\ge32$), PowerSGD fails to deliver a wall‐clock speedup in our setup due to the costly orthogonalization of an $r\times d_{\mathrm{model}}$ matrix at each iteration. In contrast, GreedyLore circumvents this overhead and consistently provides communication time savings even at elevated ranks.

\begin{table}[h]  
\renewcommand{\arraystretch}{1.5}
\begin{center}

\caption{ \small Average per‐iteration training time (in seconds) of greedy algorithms during pre‐training on the C4 dataset.}
{\tabcolsep=2pt
\begin{tabular}{lcccccc}
\toprule
Algorithm             & \textbf{LLaMA-60M}    & \textbf{LLaMA-130M}   & \textbf{LLaMA-350M}   & \textbf{LLaMA-1B}      \\
\midrule
AdamW           & 0.7396 & 1.1052 & 2.1331 & 3.7494 \\
PowerSGD        & 0.8807 & 1.3190 & 2.6841 & 4.5793 \\
GreedyLore      & \textbf{0.7279} & \textbf{1.0478} & \textbf{2.0415} & \textbf{2.7403} \\
\midrule
$r/d_\mathrm{model}$   & 32/256 & 32/512 & 32/768 & 32/1024 \\
\bottomrule
\end{tabular}
}
\label{table:c4_time}
\vspace{-5mm}
\end{center}
\end{table}

\subsection{Memory Analysis}

Figure~\ref{fig:memory} reports the peak GPU memory of GreedyLore during pre-training of LLaMA models. For models with up to 350\,M parameters, we use a per-GPU batch size of 128. For the 1B-parameter model, we reduce the batch size to 64 due to memory constraints. We do not employ gradient accumulation here, even though it enables larger effective batch sizes.

The memory overhead of GreedyLore comprises two components. First, the error-feedback mechanism requires additional storage proportional to the total number of model parameters. However, activation tensors dominate memory usage during pre-training, rendering this overhead negligible. Second, storing the projection matrix $\mathbf{U}\in\mathbb{R}^{\min(m,n)\times \min(m,n)}$ incurs storage of size $\min(m,n)^2$ by selecting either a left- or right-sided projection (see Appendix~\ref{sec:variant}). Consequently, the overall memory overhead of GreedyLore remains negligible, as confirmed by our profiling results.

\begin{figure}[h]
\vspace{-2mm}
\centering
\includegraphics[width=2.5in]{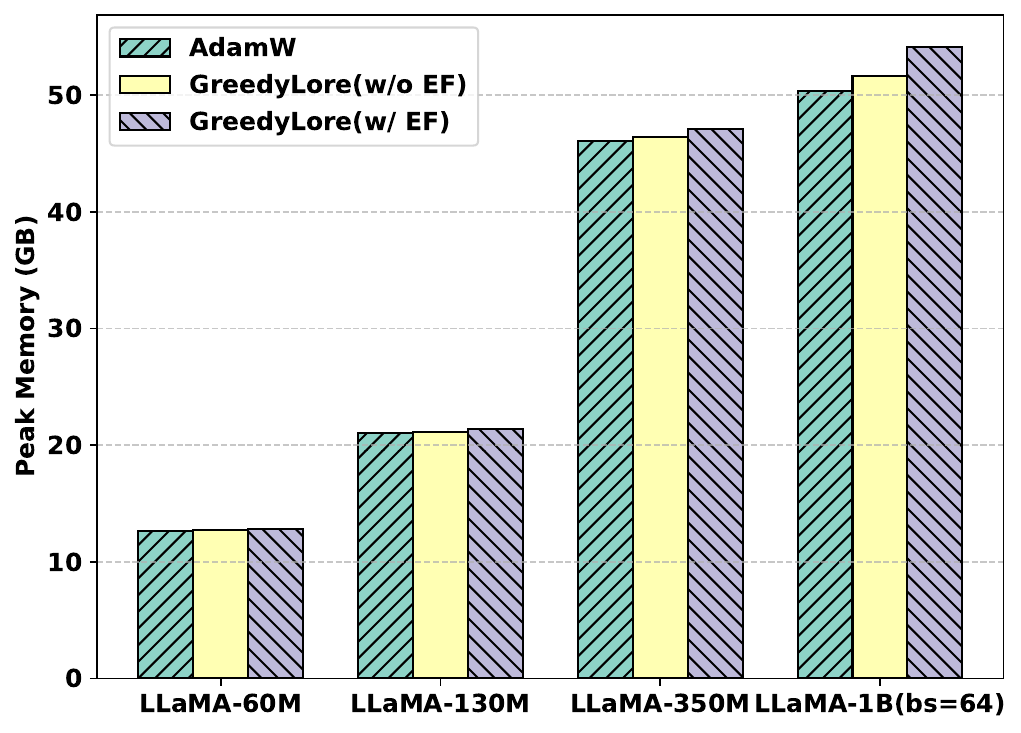}
\caption{Peak memory of pre-training of LLaMA models.}
\label{fig:memory}
\vspace{-4mm}
\end{figure}

\section{Conclusion}
In this work, we presents GreedyLore, a communication-efficient, greedy low-rank gradient compression algorithm designed to address fundamental limitations in existing distributed optimization methods. By combining a novel greedy projector selection strategy with an integrated error-feedback mechanism, GreedyLore achieves strong theoretical convergence guarantees without imposing restrictive assumptions. Our rigorous analysis confirms a convergence rate consistent with standard adaptive optimization methods. Empirical evaluations across multiple benchmark tasks, including ResNet pre-training, LLaMA model pre-training and RoBERTa fine-tuning, validate GreedyLore's superior performance relative to existing compression methods. Moreover, GreedyLore maintains negligible additional memory overhead and integrates with established distributed training frameworks, highlighting its practical utility.

\appendices
\setcounter{equation}{0}
\renewcommand{\theequation}{A\arabic{equation}}

\bibliographystyle{ieeetr}
{\footnotesize
\bibliography{reference}
}

\newpage

\onecolumn
\begin{multicols}{2}
\section{Proofs for convergence theorems}\label{app:proof}
\noindent\textbf{Notations. }We have the following notations.
\begin{itemize}
\item We let $\widetilde{\bm{G}}_t:=\frac{1}{N}\sum_{i=1}^N\nabla F(\bm{X}_t;\bm{\xi}_t^{(i)})$.
\item We let $\overline{\bm{E}}_t:=\frac{1}{N}\sum_{i=1}^N\bm{E}_t^{(i)}$.
  
\item  When applying Adam optimizer, we have the following update rules with AMSGrad-type normalization:
\begin{align}
    \bm{M}_t=&\beta_1\bm{M}_{t-1}+(1-\beta_1)\hat{\bm{G}}_t,\nonumber\\
    \bm{V}_t=&\beta_2\bm{V}_{t-1}+(1-\beta_2)(\hat{\bm{G}}_t)^2,\nonumber\\
    \widetilde{\bm{V}}_t=&\max\{\widetilde{\bm{V}}_{t-1},\|\bm{V}_t\|_{\max}\},\nonumber\\
    \bm{X}_{t+1}=&\bm{X}_t-\gamma\bm{M}_t/(\widetilde{\bm{V}}_t+\epsilon)^{1/2},\nonumber
\end{align}
where $\bm{M}_{-1}=\bm{V}_{-1}=\widetilde{\bm{V}}_{-1}=\bm{0}$. We denote $\bm{\Gamma}_t:=1/\sqrt{\widetilde{\bm{V}}_t+\epsilon}$ and $\Delta\bm{\Gamma}_t:=\bm{\Gamma}_{t-1}-\bm{\Gamma}_t$ for convenience.
\end{itemize} 

To prove the convergence theorems, we have the following lemmas, with proofs provided in the supplemental material.

\begin{lemma}\label{lm:randtopk}
    Assume ${\bm{\xi}}_i\sim\mathcal{N}(0,\sigma_i^2),i=1,2,\cdots,m$ are i.i.d. random variables, it holds that 
    \begin{align}
        &\mathbb{P}[|{\bm{\xi}}_i|\in \mathrm{Top}_k(|{\bm{\xi}}_1|,|{\bm{\xi}}_2|,\cdots,|{\bm{\xi}}_m|)]\nonumber\\
        \ge&\mathbb{P}[|{\bm{\xi}}_{j}|\in\mathrm{Top}_k(|{\bm{\xi}}_1|,|{\bm{\xi}}_2|,\cdots,|{\bm{\xi}}_m|)],\quad\ \text{if}\ \sigma_i\ge\sigma_j.
    \end{align}
\end{lemma}
\begin{lemma}[Contractive property]\label{lm:contractive}
It holds for $\hat{\bm{G}}_t$ that 
\begin{align}
    \mathbb{E}[\|\hat{\bm{G}}_t-{\bm{G}}_t\|_F^2]\le(1-\delta)\|{\bm{G}}_t\|_F^2,
\end{align}
where $\delta=r/m$.
\end{lemma}
\begin{lemma}[Constant upper bound under Assumption \ref{asp:BG}]\label{lm:BGC}
    Under Assumptions \ref{asp:stochastic} and \ref{asp:BG}, it holds that
    \begin{align*}
        \mathbb{E}\left[\left\|\frac{\beta_1}{1-\beta_1}\bm{M}_t+\overline{\bm{E}}_{t}\right\|_F^2\right]\le\left(\frac{12\beta_1^2}{(1-\beta_1)^2}+8\right)\frac{\sigma^2+B^2}{\delta^2},
    \end{align*}
    where $\delta:=r/m$.
\end{lemma}
\begin{lemma}[Constant upper bound under Assumption \ref{asp:BSG}]\label{lm:BSGC}
    Under Assumption \ref{asp:BSG}, it holds that
    \begin{align}
        \left\|\frac{\beta_1}{1-\beta_1}\bm{M}_t+\overline{\bm{E}}_{t}\right\|_F\le&\frac{\tau B_s}{1-\beta_1},\label{eq:lm-bsgc-1}\\
        \mathbb{E}\left[\left\|\frac{\beta_1}{1-\beta_1}\bm{M}_t+\overline{\bm{E}}_{t}\right\|_F^2\right]\le&\left(\frac{12\beta_1^2}{(1-\beta_1)^2}+8\right)\frac{B_s^2}{\delta^2},\label{eq:lm-bsgc-2}
    \end{align}
    where $\delta:=r/m$.
\end{lemma}
\begin{lemma}\label{lm:deltagamma}
    It holds under Assumption \ref{asp:clcu} that
    \begin{align}
        &\sum_{t=0}^{T}\mathbb{E}[\|\Delta\bm{\Gamma}_t\|_{\max}]\le2c_u,\label{eq:lm-gamma-1}\\
        &\sum_{t=0}^T\mathbb{E}[\|\Delta\bm{\Gamma}_t\|_{\max}^2]\le2c_u^2.\label{eq:lm-gamma-2}
    \end{align}
\end{lemma}

Now we provide the detailed proofs of Theorems \ref{thm:converge-adam} and \ref{thm:cov-msgd}.
\begin{proof}[Proof of Theorem \ref{thm:converge-adam}]
Let
\begin{align*}
    \bm{Y}_t=\bm{X}_t-\frac{\beta_1\gamma}{1-\beta_1}\bm{\Gamma}_{t-1}\odot\bm{M}_{t-1}-\gamma\bm{\Gamma}_{t-1}\odot\overline{\bm{E}}_{t-1},
\end{align*}
where $\bm{\Gamma}_{-1}=\frac{1}{\sqrt{\epsilon}}\bm{1}$, $\overline{\bm{E}}_{-1}=\bm{M}_{-1}=\bm{0}$. It holds that
\begin{align}
    &\bm{Y}_{t+1}-\bm{Y}_t\nonumber\\
    =&-\gamma\bm{\Gamma}_t\odot(\bm{M}_t+\beta_1(-\bm{M}_{t-1}+\hat{\bm{G}}_t)+\overline{\bm{E}}_{t}-\overline{\bm{E}}_{t-1})\nonumber\\
    &-\Delta\bm{\Gamma_t}\odot\left(\frac{\beta_1\gamma}{1-\beta_1}\bm{M}_{t-1}+\gamma\overline{\bm{E}}_{t-1}\right)\nonumber\\
    =&-\gamma\bm{\Gamma}_t\odot\widetilde{\bm{G}}_t-\Delta\bm{\Gamma_t}\odot\left(\frac{\beta_1\gamma}{1-\beta_1}\bm{M}_{t-1}+\gamma\overline{\bm{E}}_{t-1}\right),
\end{align}
where the last equality uses $\bm{M}_t=\beta_1\bm{M}_{t-1}+(1-\beta_1)\hat{\bm{G}}_t$ and $\overline{\bm{E}}_{t}=\overline{\bm{E}}_{t-1}+\widetilde{\bm{G}}_t-\hat{\bm{G}}_t$.
By $L$-smoothness, it holds that
\begin{align}
    &\mathbb{E}[f(\bm{Y}_{t+1})]-\mathbb{E}[f(\bm{Y}_t)]\nonumber\\\
    \le&\hspace{-0.7mm}-\hspace{-0.7mm}\gamma\mathbb{E}[\hspace{-0.5mm}\langle\nabla\hspace{-0.5mm} f(\hspace{-0.5mm}\bm{X}_t\hspace{-0.5mm}),\hspace{-0.5mm}\bm{\Gamma}\hspace{-0.5mm}_t\hspace{-0.5mm}\odot\hspace{-0.5mm}\widetilde{\bm{G}}_t\rangle\hspace{-0.5mm}]\hspace{-0.7mm}-\hspace{-0.7mm}\gamma\mathbb{E}[\hspace{-0.5mm}\langle\nabla\hspace{-0.5mm} f(\hspace{-0.5mm}\bm{Y}_t\hspace{-0.5mm})\hspace{-0.7mm}-\hspace{-0.7mm}\nabla\hspace{-0.5mm} f(\hspace{-0.5mm}\bm{X}_t\hspace{-0.5mm}),\hspace{-0.5mm}\bm{\Gamma}\hspace{-0.5mm}_t\hspace{-0.5mm}\odot\hspace{-0.5mm}\widetilde{\bm{G}}_t\rangle\hspace{-0.5mm}]\nonumber\\
    &-\gamma\mathbb{E}\left[\left\langle\nabla f(\bm{Y}_t),\Delta\bm{\Gamma}_t\odot\left(\frac{\beta_1}{1-\beta_1}\bm{M}_{t-1}+\overline{\bm{E}}_{t-1}\right)\right\rangle\right]\nonumber\\
    &+\frac{\gamma^2L}{2}\mathbb{E}\hspace{-0.5mm}\left[\hspace{-0.5mm}\left\|\bm{\Gamma}_t\odot\widetilde{\bm{G}}_t\hspace{-0.5mm}-\hspace{-0.5mm}\Delta\bm{\Gamma}_t\odot\left(\hspace{-0.5mm}\frac{\beta_1}{1\hspace{-0.5mm}-\hspace{-0.5mm}\beta_1}\bm{M}_{t\hspace{-0.5mm}-\hspace{-0.5mm}1}\hspace{-0.5mm}+\hspace{-0.5mm}\overline{\bm{E}}_{t\hspace{-0.5mm}-\hspace{-0.5mm}1}\hspace{-0.5mm}\right)\hspace{-0.5mm}\right\|_F^2\hspace{-0.5mm}\right]\nonumber\\
    \le&-\frac{\gamma c_l}{2}\mathbb{E}[\|\nabla f(\bm{X}_t)\|_F^2]+\frac{\gamma^2L\tau^2B_s^2}{(1-\beta_1)^2}\mathbb{E}[\|\Delta\bm{\Gamma}_t\|_{\max}^2]\nonumber\\
    &+\gamma B_s^2\left(1+\frac{\tau}{1-\beta_1}+\frac{\gamma Lc_u\tau}{1-\beta_1}\right)\mathbb{E}[\|\Delta\bm{\Gamma}_t\|_{\max}]\nonumber\\
    &+\frac{\gamma^3L^2c_u^4}{c_l}\mathbb{E}\hspace{-0.5mm}\left[\hspace{-0.5mm}\left\|\hspace{-0.5mm}\frac{\beta_1}{1\hspace{-0.5mm}-\hspace{-0.5mm}\beta_1}\bm{M}_{t\hspace{-0.5mm}-\hspace{-0.5mm}1}\hspace{-0.5mm}+\hspace{-0.5mm}\overline{\bm{E}}_{t\hspace{-0.5mm}-\hspace{-0.5mm}1}\hspace{-0.5mm}\right\|_F^2\hspace{-0.5mm}\right]\hspace{-0.5mm}+\hspace{-0.5mm}\frac{\gamma^2Lc_u^2\sigma^2}{n},\label{eq:adam-descent}
\end{align}
where the second inequality uses
\begin{align}
    &\mathbb{E}[\langle\nabla f(\bm{Y}_t)-\nabla f(\bm{X}_t),\bm{\Gamma}_t\odot\widetilde{\bm{G}}_t\rangle]\nonumber\\
    =&\mathbb{E}[\langle\nabla f(\bm{Y}_t)-\nabla f(\bm{X}_t),\bm{\Gamma}_{t-1}\odot\nabla f(\bm{X}_t)\rangle]\nonumber\\
    &-\mathbb{E}[\langle\nabla f(\bm{Y}_t)-\nabla f(\bm{X}_t),\Delta\bm{\Gamma}_t\odot\widetilde{\bm{G}}_t\rangle]\nonumber\\
    \le&\frac{c_l}{4}\mathbb{E}[\|\nabla f(\bm{X}_t)\|_F^2]+\frac{c_u^2}{c_l}\mathbb{E}[\|\nabla f(\bm{Y}_t)-\nabla f(\bm{X}_t)\|_F^2]\nonumber\\
    &+\gamma L\mathbb{E}\left[\left\|\bm{\Gamma}_t\hspace{-0.5mm}\odot\hspace{-0.5mm}\left(\frac{\beta_1}{1-\beta_1}\bm{M}_{t-1}\hspace{-0.5mm}+\hspace{-0.5mm}\overline{\bm{E}}_{t-1}\right)\right\|_F\|\Delta\bm{\Gamma}_t\odot\widetilde{\bm{G}}_t\|_F\right]\nonumber\\
    \le&\frac{c_l}{4}\mathbb{E}[\|\nabla f(\bm{X}_t)\|_F^2]+\frac{\gamma^2L^2c_u^4}{c_l}\mathbb{E}\left[\left\|\frac{\beta_1}{1-\beta_1}\bm{M}_{t-1}+\overline{\bm{E}}_{t-1}\right\|_F^2\right]\nonumber\\
    &+\frac{\gamma Lc_u\tau B_s^2}{1-\beta_1}\mathbb{E}[\|\Delta\bm{\Gamma}_t\|_{\max}],\nonumber
\end{align}and $\gamma\le c_l/(4Lc_u^2)$. 
Summing \eqref{eq:adam-descent} from $t=0$ to $T$ and apply Lemma \ref{lm:BSGC} yields
\begin{align}
    &\frac{1}{T+1}\sum_{t=0}^{T}\mathbb{E}[\|\nabla f(\bm{X}_t)\|_F^2]\nonumber\\
    \le&\left(B_s^2\hspace{-0.5mm}+\hspace{-0.5mm}\frac{\tau B_s^2}{1\hspace{-0.5mm}-\hspace{-0.5mm}\beta_1}\hspace{-0.5mm}+\hspace{-0.5mm}\frac{\gamma Lc_u\tau B_s^2}{1\hspace{-0.5mm}-\hspace{-0.5mm}\beta_1}\right)\hspace{-0.5mm}\cdot\hspace{-0.5mm}\frac{2}{c_l(T\hspace{-0.5mm}+\hspace{-0.5mm}1)}\sum_{t=0}^T\mathbb{E}[\|\Delta\bm{\Gamma}_t\|_{\max}]\nonumber\\
    &+\frac{2[f(\bm{X}_0)\hspace{-0.5mm}-\hspace{-0.5mm}\mathbb{E}[f(\bm{Y}_{T+1})]]}{\gamma c_l(T\hspace{-0.5mm}+\hspace{-0.5mm}1)}\hspace{-0.5mm}+\hspace{-0.5mm}\left(\hspace{-0.5mm}\frac{12\beta_1^2}{(1\hspace{-0.5mm}-\hspace{-0.5mm}\beta_1)^2}\hspace{-0.5mm}+\hspace{-0.5mm}8\hspace{-0.5mm}\right)\hspace{-0.5mm}\frac{2\gamma^2L^2c_u^4B_s^2}{c_l^2\delta}\nonumber\\
    &+\frac{2\gamma L\tau^2 B_s^2}{(1-\beta_1)^2c_l}\cdot\frac{1}{T+1}\sum_{t=0}^T\mathbb{E}[\|\Delta\bm{\Gamma}_t\|_{\max}^2]+\frac{2\gamma Lc_u^2\sigma^2}{nc_l}.\nonumber
\end{align}
Further applying Lemma \ref{lm:deltagamma} achieves the desired inequality.
\end{proof}

\begin{proof}[Proof of Theorem \ref{thm:cov-msgd}]
    Following the proof of \eqref{eq:adam-descent} in Theorem \ref{thm:converge-adam}, while substituting $\Delta\bm{\Gamma}_t$ with $\bm{0}$, 
 $\bm{\Gamma}_t$ with $\bm{1}$, $c_u$ with $1$, $c_l$ with 1, and Lemma \ref{lm:BSGC} with Lemma \ref{lm:BGC}, we obtain
\begin{align}
    \mathbb{E}[f(\bm{Y}_{t+1})]-\mathbb{E}[f(\bm{Y}_t)]\le&-\frac{\gamma}{2}\mathbb{E}[\|\nabla f(\bm{X}_t)\|_F^2]+\frac{\gamma^2 L\sigma^2}{n}\nonumber\\
    &\hspace{-12.5mm}+\hspace{-0.5mm}\left(\hspace{-0.5mm}\frac{12\beta_1^2}{(1\hspace{-0.5mm}-\hspace{-0.5mm}\beta_1)^2}\hspace{-0.5mm}+\hspace{-0.5mm}8\hspace{-0.5mm}\right)\hspace{-0.5mm}\frac{\gamma^3L^2(B^2\hspace{-0.5mm}+\hspace{-0.5mm}\sigma^2)}{\delta},\label{eq:msgd-decent}
\end{align}
Summing \eqref{eq:msgd-decent} from $t=0$ to $T$ yields \eqref{eq:thm-msgd}.
\end{proof}

\section{Missing proofs}\label{app:missing-proof}

\begin{proof}[Proof of Lemma \ref{lm:randtopk}]
    WLOG assume $i=1$, $j=2$ and $\sigma_1\ge\sigma_2$. It suffices to prove that for any given ${\bm{\xi}}_3,\cdots,{\bm{\xi}}_m\in\mathbb{R}$, the conditional probabilities satisfy $\mathbb{P}_A[\Lambda_1]\ge\mathbb{P}_A[\Lambda_2]$, where $\mathbb{P}_A[\cdot]:=\mathbb{P}[\cdot|{\bm{\xi}}_3,\cdots,{\bm{\xi}}_m]$, and $\Lambda_i:=\{|{\bm{\xi}}_i|\in\mathrm{Top}_k(|{\bm{\xi}}_1|,\cdots,|{\bm{\xi}}_m|)\}$. WLOG assume $|{\bm{\xi}}_3|\le|{\bm{\xi}}_4|\le\cdots\le|{\bm{\xi}}_m|$. If $k=m$ or $\sigma_2=0$, the result is trivial. In the following we assume $k<m$ and $\sigma_2>0$.
    If $k=m-1$, it holds that 
    \begin{align}
        \mathbb{P}_A[\Lambda_1\backslash\Lambda_2]   =&\mathbb{P}_A[|{\bm{\xi}}_1|\ge|{\bm{\xi}}_3|]\cdot\mathbb{P}_A[|{\bm{\xi}}_2|\le|{\bm{\xi}}_3|]\nonumber\\
        &+\int_{0}^{|{\bm{\xi}}_3|}\int_{x}^{|{\bm{\xi}}_3|}\frac{2}{\pi\sigma_1\sigma_2}\cdot e^{-\frac{x^2}{2\sigma_2^2}-\frac{y^2}{2\sigma_1^2}}\mathrm{d}y\ \mathrm{d}x\nonumber\\
        \ge&\mathbb{P}_A[|{\bm{\xi}}_2|\ge|{\bm{\xi}}_3|]\cdot\mathbb{P}_A[|{\bm{\xi}}_1|\le|{\bm{\xi}}_3|]\nonumber\\
        &+\int_0^{|{\bm{\xi}}_3|}\int_x^{|{\bm{\xi}}_3|}\frac{2}{\pi\sigma_1\sigma_2}\cdot e^{-\frac{x^2}{2\sigma_1^2}-\frac{y^2}{2\sigma_2^2}}\mathrm{d}y\ \mathrm{d}x\nonumber\\
        =&\mathbb{P}_A[\Lambda_2\backslash\Lambda_1],
    \end{align}
    which implies $\mathbb{P}_A[\Lambda_1]\ge\mathbb{P}_A[\Lambda_2]$. If $2\le k\le m-2$, it holds that
    \begin{align}
        \mathbb{P}_A[\Lambda_1\backslash\Lambda_2]=&\mathbb{P}_A[|{\bm{\xi}}_1|\ge|{\bm{\xi}}_{m-k+1}|]\cdot\mathbb{P}[|{\bm{\xi}}_2|\le|{\bm{\xi}}_{m-k+1}|]\nonumber\\
        &+\int_{|{\bm{\xi}}_{m-k+1}|}^{|{\bm{\xi}}_{m-k+2}|}\int_x^{+\infty}\frac{2}{\pi\sigma_1\sigma_2}\cdot e^{-\frac{x^2}{2\sigma_2^2}-\frac{y^2}{2\sigma_1^2}}\mathrm{d}y\ \mathrm{d}x\nonumber\\
        \ge&\mathbb{P}_A[|{\bm{\xi}}_2|\ge|{\bm{\xi}}_{m-k+1}|]\cdot\mathbb{P}[|{\bm{\xi}}_2|\le|{\bm{\xi}}_{m-k+1}|]\nonumber\\
        &+\int_{|{\bm{\xi}}_{m-k+1|}}^{|{\bm{\xi}}_{m-k+2|}}\int_x^{+\infty}\frac{2}{\pi\sigma_1\sigma_2}\cdot e^{-\frac{x^2}{2\sigma_1^2}-\frac{y^2}{2\sigma_2^2}}\mathrm{d}y\ \mathrm{d}x\nonumber\\
        =&\mathbb{P}_A[\Lambda_2\backslash\Lambda_1],
    \end{align}
    which implies $\mathbb{P}_A[\Lambda_1]\ge\mathbb{P}_A[\Lambda_2]$. If $k=1$, it holds that
    \begin{align}
        \mathbb{P}_A[\Lambda_1\backslash\Lambda_2]=&\mathbb{P}_A[|{\bm{\xi}}_1|\ge|{\bm{\xi}}_m|]\cdot\mathbb{P}_A[|{\bm{\xi}}_2|\le|{\bm{\xi}}_m|]\nonumber\\
        &+\int_{|{\bm{\xi}}_m|}^{+\infty}\int_x^{+\infty}\frac{2}{\pi\sigma_1\sigma_2}\cdot e^{-\frac{x^2}{2\sigma_2^2}-\frac{y^2}{2\sigma_1^2}}\mathrm{d}y\ \mathrm{d}x\nonumber\\
        \ge&\mathbb{P}_A[|{\bm{\xi}}_2|\ge|{\bm{\xi}}_m|]\cdot\mathbb{P}_A[|{\bm{\xi}}_1|\le|{\bm{\xi}}_m|]\nonumber\\
        &+\int_{|{\bm{\xi}}_m|}^{+\infty}\int_x^{+\infty}\frac{2}{\pi\sigma_1\sigma_2}\cdot e^{-\frac{x^2}{2\sigma_1^2}-\frac{y^2}{2\sigma_2^2}}\mathrm{d}y\ \mathrm{d}x\nonumber\\
        =&\mathbb{P}_A[\Lambda_2\backslash\Lambda_1].
    \end{align}
    which implies $\mathbb{P}_A[\Lambda_1]\ge\mathbb{P}[\Lambda_2]$.
\end{proof}
\begin{proof}[Proof of Lemma \ref{lm:contractive}]
    When $\mathsf{mod}(t,\tau)=0$, we have $\mathbb{E}[\|\hat{\bm{G}}_t-{\bm{G}}_t\|_F^2]=0\le(1-\delta)\|{\bm{G}}_t\|_F^2$ directly. In the following suppose $\mathsf{mod}(t,\tau)\ne0$. Let $\bm{u}_j$ denote the $j$-th column of $\bm{U}$, and $l_j:=\|\bm{u}_j^\top {\bm{G}}_t\|_2$. It holds that $\overline{\bm{\Lambda}}_j\sim\mathcal{N}(0,l_j^2)$. Consider $\pi_1,\pi_2,\cdots,\pi_m$ the permutation of $\{1,2\,\cdots,m\}$ satisfying $l_{\pi_1}\ge l_{\pi_2}\ge\cdots\ge l_{\pi_m}$, by Lemma \ref{lm:randtopk} we have $p_{\pi_1}\ge p_{\pi_2}\ge\cdots\ge p_{\pi_m}$, where $p_j:=\mathbb{P}[\bm{u}_j\text{ is selected in }\bm{P}_t]$. Noting $\sum_{j=1}^ml_j^2=\|\bm{U}^\top{\bm{G}}_t\|_F^2=\|{\bm{G}}_t\|_F^2$, we have
    \begin{align}
        &\mathbb{E}[\|\hat{\bm{G}}_t-{\bm{G}}_t\|_F^2]=\mathbb{E}[\|(\bm{I}_m-\bm{P}_t\bm{P}_t^\top){\bm{G}}_t\|_F^2]\nonumber\\
        =&\sum_{j=1}^m(1-p_j)\|\bm{u}_j^\top {\bm{G}}_t\|_F^2=\|{\bm{G}}_t\|_F^2-\sum_{j=1}^mp_{\pi_j}l_{\pi_j}^2\nonumber\\
        \ge&\|{\bm{G}}_t\|_F^2-\frac{1}{m}\sum_{j=1}^mp_{\pi_j}\sum_{j=1}^ml_{\pi_j}^2=\left(1-\frac{r}{m}\right)\|{\bm{G}}_t\|_F^2,\nonumber
    \end{align}
    where the inequality uses Chebyshev's Inequality, and the last equality uses $\sum_{j=1}^mp_j=r$.
\end{proof}
\begin{proof}[Proof of Lemma \ref{lm:BGC}]
    To bound $\mathbb{E}[\|\overline{\bm{E}}_{t}\|_F^2]$, we have
    \begin{align}
        \mathbb{E}[\|\overline{\bm{E}}_{t}\|_F^2]=&\mathbb{E}[\|(\bm{I}_m-\bm{P}_{t}\bm{P}_{t}^\top){\bm{G}}_{t}\|_F^2]\nonumber\\
        \le&(1-\delta)\mathbb{E}[\|{\bm{G}}_{t}\|_F^2],\label{eq:pflm-bgc-1}
    \end{align}
    when $\mathsf{mod}(t,\tau)\ne0$, where the inequality uses Lemma \ref{lm:contractive}. When $\mathsf{mod}(t,\tau)=0$, the same inequality holds trivially.
    To bound $\mathbb{E}[\|{\bm{G}}_t\|_F^2]$, we have
    \begin{align}
        &\mathbb{E}[\|{\bm{G}}_t\|_F^2]=\mathbb{E}\left[\left\|\frac{1}{N}\sum_{i=1}^N\nabla F_i(\bm{X}_t;{\bm{\xi}}_t^{(i)})+\overline{\bm{E}}_{t-1}\right\|_F^2\right]\nonumber\\
        =&\mathbb{E}[\|\nabla\hspace{-0.5mm} f(\hspace{-0.5mm}\bm{X}\hspace{-0.5mm}_t\hspace{-0.5mm})\hspace{-0.5mm}+\hspace{-0.5mm}\overline{\bm{E}}\hspace{-0.5mm}_{t\hspace{-0.5mm}-\hspace{-0.5mm}1}\|_F^2]\hspace{-0.5mm}+\hspace{-0.5mm}\mathbb{E}\hspace{-0.5mm}\left[\hspace{-0.5mm}\left\|\frac{1}{N}\hspace{-0.5mm}\sum_{i=1}^N\hspace{-0.5mm}\nabla\hspace{-0.5mm} F_i(\hspace{-0.5mm}\bm{X}\hspace{-0.5mm}_t;\hspace{-0.5mm}{\bm{\xi}}_t^{(i)}\hspace{-0.5mm})\hspace{-0.5mm}-\hspace{-1.0mm}\nabla\hspace{-0.5mm} f(\hspace{-0.5mm}\bm{X}\hspace{-0.5mm}_t\hspace{-0.5mm})\hspace{-0.5mm}\right\|_F^2\hspace{-0.5mm}\right]\nonumber\\
        \le&\left(1+\frac{2}{\delta}\right)\mathbb{E}[\|\nabla f(\bm{X}_t)\|_F^2]+\left(1+\frac{\delta}{2}\right)\mathbb{E}[\|\overline{\bm{E}}_{t-1}\|_F^2]+\frac{\sigma^2}{N}\nonumber\\
        \le&\left(1+\frac{2}{\delta}\right)B^2+\left(1-\frac{\delta}{2}\right)\mathbb{E}[\|{\bm{G}}_{t-1}
        \|_F^2]+\frac{\sigma^2}{N}, \label{eq:pflm-bgc-2}
    \end{align}
    where the first inequality uses Young's Inequality and Assumption \ref{asp:stochastic}, the second inequality uses \eqref{eq:pflm-bgc-1} and Assumption \ref{asp:BG}. 
    Noting $\mathbb{E}[\|{\bm{G}}_0\|]_F^2]\le B^2+\sigma^2/N$, \eqref{eq:pflm-bgc-2} indicates that 
    \begin{align}
        \mathbb{E}[\|\overline{\bm{G}}_t\|_F^2]\le\frac{4+2\delta}{\delta^2}B^2+\frac{2\sigma^2}{N\delta},\label{eq:pflm-bgc-3}
    \end{align}
    which further implies
    \begin{align}
        \mathbb{E}[\|\overline{\bm{E}}_{t}\|_F^2]\overset{\eqref{eq:pflm-bgc-1}}{\le}&(1-\delta)\mathbb{E}[\|{\bm{G}}_{t}\|_F^2]
        \nonumber\\
        \overset{\eqref{eq:pflm-bgc-3}}{\le}&\frac{4-2\delta-2\delta^2}{\delta^2}B^2+\frac{2(1-\delta)\sigma^2}{N\delta},\label{eq:pflm-bgc-4}
    \end{align}
    and 
    \begin{align}
        \mathbb{E}[\|\bm{M}_t\|_F^2]=&\mathbb{E}[\|\beta_1\bm{M}_{t-1}+(1-\beta_1)\hat{\bm{G}}_t\|_F^2]\nonumber\\
        \le&\beta_1\mathbb{E}[\|\bm{M}_{t-1}\|_F^2]+(1-\beta_1)\mathbb{E}[\|{\bm{G}}_t\|_F^2],\label{eq:pflm-bgc-5}
    \end{align}
    where the inequality uses Jensen's Inequality and $\|\hat{\bm{G}}_t\|_F=\|\bm{P}_t\bm{P}_t^\top{\bm{G}}_t\|_F\le\|{\bm{G}}_t\|_F$. Combining \eqref{eq:pflm-bgc-3}\eqref{eq:pflm-bgc-5} and the fact that $\|\bm{M}_{-1}\|_F^2=0$ yields
    \begin{align}
        \mathbb{E}[\|\bm{M}_t\|_F^2]\le\frac{4+2\delta}{\delta^2}B^2+\frac{2\sigma^2}{N\delta},\label{eq:pflm-bgc-6}
    \end{align}
    thus we have
    \begin{align}
        &\mathbb{E}\left[\left\|\frac{\beta_1}{1-\beta_1}\bm{M}_t+\overline{\bm{E}}_{t}\right\|_F^2\right]\nonumber\\
        \le&\frac{2\beta_1^2}{(1-\beta_1)^2}\mathbb{E}[\|\bm{M}_t\|_F^2]+2\mathbb{E}[\|\overline{\bm{E}}_{t}\|_F^2]\nonumber\\
        \overset{\eqref{eq:pflm-bgc-4}\eqref{eq:pflm-bgc-6}}{\le}&\left(\frac{12\beta_1^2}{(1-\beta_1)^2}+8\right)\frac{B^2+\sigma^2}{\delta^2},\nonumber
\end{align}
which completes the proof.
\end{proof}
\begin{proof}[Proof of Lemma \ref{lm:BSGC}]
    We first bound $\|\overline{\bm{E}}_t\|_F$ as follows:
    \begin{align}
        &\|\overline{\bm{E}}_{k\tau}\|_F=0,\text{ and }\nonumber\\
        &\|\overline{\bm{E}}_{k\tau+s}\|_F=\|(\bm{I}_m-\bm{P}_{k\tau+s}\bm{P}_{k\tau+s}^\top){\bm{G}}_{k\tau+s}\|_F\nonumber\\
        &\hspace{16mm}\le\|\overline{\bm{E}}_{k\tau+s-1}\|_F+B_s,\;s=1,2,\cdots,\tau-1.\nonumber\\
        \Rightarrow&\|\overline{\bm{E}}_t\|_F\le(\tau-1) B_s,\;\forall\;t\in\mathbb{N}.\label{eq:pflm-bsgc-1}
\end{align}
\eqref{eq:pflm-bsgc-1} implies that
\begin{align}
    \|{\bm{G}}_t\|_F\le \|\overline{\bm{E}}_{t-1}\|_F+B_s\le \tau B_s.\label{eq:pflm-bsgc-2}
\end{align}
Consequently, we have
\begin{align}
    \|\bm{M}_t\|=&\|\beta_1\bm{M}_{t-1}+(1-\beta_1)\hat{\bm{G}}_t\|_F\nonumber\\
    \le&\beta_1\|\bm{M}_{t-1}\|_F+(1-\beta_1)\|{\bm{G}}_t\|_F\nonumber\\
    \overset{\eqref{eq:pflm-bsgc-2}}{\le}&\beta_1\|\bm{M}_{t-1}\|_F+(1-\beta_1)\tau B_s\nonumber\\
    \overset{\|\bm{M}_{-1}\|_F=0}{\Longrightarrow}&\|\bm{M}_t\|_F\le\tau B_s.\label{eq:pflm-bsgc-3}
\end{align}
Combining \eqref{eq:pflm-bsgc-1}\eqref{eq:pflm-bsgc-3} achieves \eqref{eq:lm-bsgc-1}. \eqref{eq:lm-bsgc-2} can be proved by following the proof of Lemma \ref{lm:BGC} except for bounding $\mathbb{E}[\|{\bm{G}}_t\|_F^2]$ by 
\begin{align}
    \mathbb{E}[\|\overline{\bm{G}}_t\|_F^2]=&\mathbb{E}\left[\left\|\frac{1}{N}\sum_{i=1}^N\nabla F_i(\bm{X}_t;{\bm{\xi}}_t^{(i)})+\overline{\bm{E}}_{t-1}\right\|_F^2\right]\nonumber\\
    \le&\left(1+\frac{2}{\delta}\right)B_s^2+\left(1+\frac{\delta}{2}\right)\mathbb{E}[\|\overline{\bm{E}}_{t-1}\|_F^2]\nonumber\\
    \le&\left(1+\frac{2}{\delta}\right)B_s^2+\left(1-\frac{\delta}{2}\right)\mathbb{E}[\|{\bm{G}}_{t-1}\|_F^2],\nonumber
\end{align}
which completes the proof.
\end{proof}

\begin{proof}[Proof of Lemma \ref{lm:deltagamma}]
    By the update rule of $\widetilde{\bm{V}}_t$ we know that
    \begin{align}
        &\max_{1\le i\le m \atop 1\le j\le n}\{\widetilde{{{V}}}_{t,i,j}\}\le\min_{1\le i\le m\atop1\le j\le n}\{\widetilde{{{V}}}_{t+1,i,j}\}\nonumber\\
        \Rightarrow&\min_{1\le i\le m\atop 1\le j\le n}\{{{\Gamma}}_{t,i,j}\}\ge\max_{1\le i\le m\atop1\le j\le n}\{{{\Gamma}}_{t+1,i,j}\}.\label{eq:pflm-gamma-1}
    \end{align}
    Thus, we have
    \begin{align*}
        \|\Delta\bm{\Gamma}_t\|_{\max}\hspace{-0.5mm}=&\|\bm{\Gamma}_{t\hspace{-0.5mm}-\hspace{-0.5mm}1}\hspace{-0.5mm}-\hspace{-0.5mm}\bm{\Gamma}_t\|_{\max}\nonumber\hspace{-0.5mm}\le\hspace{-0.7mm}\max_{1\le i\le m\atop1\le j\le n}\hspace{-0.7mm}\{\hspace{-0.5mm}{{\Gamma}}_{t\hspace{-0.5mm}-\hspace{-0.5mm}1,i,j}\hspace{-0.5mm}\}\hspace{-0.5mm}-\hspace{-1.5mm}\min_{1\le i\le m\atop1\le j\le n}\hspace{-0.7mm}\{\hspace{-0.5mm}{{\Gamma}}_{t,i,j}\hspace{-0.5mm}\}\nonumber\\
        \le&\left(\max_{1\le i\le m\atop 1\le j\le n}\{{\Gamma}_{t-1,i,j}\}-\max_{1\le i\le m\atop 1\le j\le n}\{{\Gamma}_{t,i,j}\}\right)\nonumber\\
        &+\left(\min_{1\le i\le m\atop 1\le j\le n}\{{\Gamma}_{t-1,i,j}\}-\min_{1\le i\le m\atop1\le j\le n}\{{\Gamma}_{t,i,j}\}\right),
    \end{align*}
    where the last inequality uses \eqref{eq:pflm-gamma-1}, and similarly,
    \begin{align}
        &\|\Delta\bm{\Gamma}_t\|^2_{\max}\le\|\bm{\Gamma}_{t-1}^2-\bm{\Gamma}_t^2\|_{\max}\nonumber\\
        &\le\left(\max_{1\le i\le m\atop1\le j\le n}\{{\Gamma}_{t-1,i,j}^2\}-\max_{1\le i\le m\atop1\le j\le n}\{{\Gamma}_{t,i,j}^2\}\right)\nonumber\\
        &+\left(\min_{1\le i\le m\atop 1\le j\le n}\{{\Gamma}_{t-1,i,j}^2\}-\min_{1\le i\le m\atop1\le j\le n}\{{\Gamma}_{t,i,j}^2\}\right)\label{eq:pflm-gamma-2}.
    \end{align}
    Summing \eqref{eq:pflm-gamma-1} or \eqref{eq:pflm-gamma-2} from $t=0$ to $T$ and applying Assumption \ref{asp:clcu} yields \eqref{eq:lm-gamma-1} or \eqref{eq:lm-gamma-2}, respectively.
\end{proof}

\section{Expanded example of non-contractive compressor}
\label{sec:expaned_counter_example}

In this section, we present a simple two‐dimensional quadratic example that demonstrates the failure of the contraction property under a rank‐one compression operator.  We consider the scalar parameter $L>0$ and define:
\[
f : \mathbb{R}^{2\times 2} \;\to\; \mathbb{R}, 
\qquad
f\bigl(\operatorname{diag}(x,y)\bigr)
= \frac{L\,x^2}{2} \;+\; \frac{L\,y^2}{4},
\]
where $\operatorname{diag}(x,y)\in\mathbb{R}^{2\times2}$ denotes the diagonal matrix with entries $x,y\in\mathbb{R}$.  It is straightforward to verify that $f$ is $L$-smooth and that its gradient (with respect to the Frobenius inner product) satisfies
\begin{equation}
\label{eq:gradient_f}
\nabla f\bigl(\operatorname{diag}(x,y)\bigr)
= \begin{pmatrix}
L\,x & 0 \\[0.5em]
0 & \tfrac{L}{2}\,y
\end{pmatrix}.
\end{equation}
At the point \((1,1)\) we have
\[
\nabla f(1,1)
= \begin{pmatrix}
L & 0\\
0 & \tfrac{L}{2}
\end{pmatrix}.
\]
Since \(\bm G_t\) has equal singular values, let the SVD arbitrarily select the first singular vector \(e_1\). Hence the rank‐one compressor fixes the one‐dimensional subspace
\[
\bm P = e_1 = \begin{pmatrix}1\\0\end{pmatrix},
\]
so that only the first row of the gradient is retained:
\[
\mathcal C\bigl(\nabla f(1,1)\bigr)
= \bm P\,\bm P^\top\,
\begin{pmatrix}
L & 0\\
0 & \tfrac{L}{2}
\end{pmatrix}
= \begin{pmatrix}L & 0\\0 & 0\end{pmatrix}.
\]
Now we perform \(\tau\) steps of (compressed) gradient descent with step size \(\gamma=\tfrac1{2L}\).  At each step,
\[
\begin{pmatrix}x_{k+1}\\y_{k+1}\end{pmatrix}
=
\begin{pmatrix}x_k\\y_k\end{pmatrix}
- \gamma\,
\begin{pmatrix}Lx_k & 0\\0 & 0\end{pmatrix}
\begin{pmatrix}1\\1\end{pmatrix}
=
\begin{pmatrix}\tfrac12 x_k\\y_k\end{pmatrix},
\]
then after \(\tau\) iterations,
\[
(x_\tau,y_\tau)
= \bigl(2^{-\tau},\,1\bigr).
\]
At this point, the true gradient is
\[
\nabla f(x_\tau,y_\tau)
= \begin{pmatrix}
L\,2^{-\tau} & 0\\
0 & \tfrac{L}{2}
\end{pmatrix},
\]
but the compressor still projects onto the first axis:
\[
\|\mathcal C(\nabla f(x_\tau,y_\tau))\|_F^2
= \bigl(L\,2^{-\tau}\bigr)^2
= L^2\,2^{-2\tau},
\]
\[
\|\nabla f(x_\tau,y_\tau)\|_F^2
= \bigl(L\,2^{-\tau}\bigr)^2 + \bigl(\tfrac{L}{2}\bigr)^2
= L^2\Bigl(2^{-2\tau} + \tfrac14\Bigr).
\]
Hence,
\[
\alpha
= \frac{\|\mathcal C(\nabla f(x_\tau,y_\tau))\|_F^2}{\|\nabla f(x_\tau,y_\tau)\|_F^2}
= \frac{L^2\,2^{-2\tau}}{L^2\bigl(2^{-2\tau} + \tfrac14\bigr)}
= \frac{2^{-2\tau}}{2^{-2\tau} + \tfrac14}.
\]
However, since \(\tau\) is typically on the order of hundreds, the term \(2^{-2\tau}\) becomes astronomically small.  For example, when \(\tau>100\), \(2^{-2\tau} < 2^{-200} \approx 6.2\times10^{-61},\) which is well below the smallest positive normalized value in IEEE-754 single-precision arithmetic (approximately \(1.18\times10^{-38}\)).  Consequently, \(2^{-2\tau}\) underflows to zero in FP32, allowing us to assume \(2^{-2\tau} = 0\) and thus \(\alpha = 0\), in direct contradiction of the contraction property required by Assumption~\ref{compressor-assumption}.
\vspace{2mm}

\begin{algorithm}[H]
\caption{GreedyLore Compressor with parameters of arbitrary two-dimension shape}
\label{alg:greedylore_variant}
\vspace{1pt}
\JustifyAlgo{\hspace{-4.9mm} {\bfseries Input}: $N$ nodes, number of total iterations $T$, subspace changing frequency $\tau$, rank $r$, initial weight $\bm{X}_{0} \in \mathbb{R}^{m \times n}$ and initial error buffer $\bm{E}_{-1}^{(i)} = \bm{0} \in \mathbb{R}^{\min(m,n) \times \max(m, n)}$ for node $i \in [N]$.} \\
\vspace{2pt}
\textbf{Output}: Sequence of model weights $\{\bm{X}_t\}_{t=0}^{T+1}$.

\For{$t=0,\ldots,T$}{
    \textbf{(On $i$-th node)} \\
  
    $\bm{G}_t^{(i)} \gets \begin{cases}
      \nabla F_i({\bm{X}}_t; {\bm{\xi}}_t^{(i)}) + \bm{E}_{t-1}^{(i)}, &m \le n, \\[1ex]
       \nabla F_i({\bm{X}}_t; {\bm{\xi}}_t^{(i)})^\top + \bm{E}_{t-1}^{(i)}, &m > n.
    \end{cases}$ \\
    $\bm{P}_t, \bm{U} \gets \mathsf{\bf Semi}\mbox{-}\mathsf{\bf Lazy}\mbox{-}\mathsf{\bf SVD}(\{\bm{G}_t^{(i)}\}_{i=1}^N, \bm{U}, t)$. \\
    $\bm{R}_t^{(i)} \gets {\bm{P}_t}^\top{\bm{G}}_t^{(i)}$. \\
    $\bm{E}_t^{(i)} \gets \bm{G}_t^{(i)} - \bm{P}_t\bm{R}_t^{(i)}$. \\
    ${\bm{R}}_t \gets \frac{1}{N}\sum_{i=1}^{N}{\bm{R}}_t^{(i)}$. \hfill (All-Reduce)   \\
    $\hat{\bm{G}}_t \gets {\bm{P}_t}{\bm{R}}_t$. \\
  
    $\bm{X}_{t+1} \gets \begin{cases}
        \mathsf{\bf Optimizer}(\bm{X}_{t}, \hat{\bm{G}}_t^\top, \gamma), &m \le n, \\[1ex]
        \mathsf{\bf Optimizer}(\bm{X}_{t}, \hat{\bm{G}}_t, \gamma), &m > n.
    \end{cases}$ \\
}
\Return $\{\bm{X}_{t}\}_{t=0}^{T+1}$.

\vspace{2mm}
\sub{$\mathsf{\bf Semi}\mbox{-}\mathsf{\bf Lazy}\mbox{-}\mathsf{\bf SVD}(\{\bm{G}_t^{(i)}\}_{i=1}^N, \bm{U}, t)$}{
    \uIf{$\mathsf{mod}(t, \tau) = 0$}{
        $\bm{G}_t \gets \frac{1}{N}\sum_{i=1}^{N}{\bm{G}_t}^{(i)} $.\hfill (All-Reduce) \\
        $\bm{U},\bm{\Sigma}_t, \bm{V}_t \gets \mathrm{\bf SVD}(\bm{G}_t)$. \\
        \vspace{-0.5mm}
        \Return $\bm{U}_{:,:r}, \bm{U}$.
    }\Else{
        $\bm{P}_{t} \gets \mathsf{\bf Approx}\mbox{-}\mathsf{\bf Top}\mbox{-}\mathsf{\bf r}(\{{\bm G}_t^{(i)}\}_{i=1}^N, {\bm U}, r)$. \\
        \Return $\bm{P}_{t}, \bm{U}$. \\
    }
}
\end{algorithm}

\section{Detailed Implementations of GreedyLore}\label{sec:variant}

In practical neural‐network training scenarios, we often encounter gradient matrices $\bm{G}_t^{(i)}\in\mathbb{R}^{m\times n}$ with shape $m>n$. In such case, storing a full projection $\bm{U}\in\mathbb{R}^{m\times m}$ incurs an $\mathcal{O}(m^2)$ memory cost. When $m$ is much larger than $n$, this $\mathcal{O}(m^2)$ memory cost can be much larger than gradient memory with order $\mathcal{O}(nm)$, inducing prohibitive memory cost.

In fact, this memory cost can be lowered down to $\mathcal{O}(\min{(n, m)}^2)$ with simple modification. The matrix with shape of $\mathbb{R}^{m \times n}$ can be compressed to low-rank by either left matrix multiplication or right matrix multiplication. When applying left multiplication, we select $\mathbb{R}^{r \times m}$ projector from $\mathbb{R}^{m \times m}$ matrix and compress the gradient into shape $\mathbb{R}^{r \times n}$, leading to $\mathcal{O}(m^2)$ memory cost and $\mathcal{O}(nr)$ communication cost. By contrast, when applying left multiplication, we select $\mathbb{R}^{n \times r}$ projector from $\mathbb{R}^{n \times n}$ matrix and compress the gradient into shape $\mathbb{R}^{m \times r}$, leading to $\mathcal{O}(n^2)$ memory cost and $\mathcal{O}(mr)$ communication cost. Since the shape of gradient is fixed all the time, we can select the multiplication type according to the value $m, n$ to avoid potential prohibitive memory cost.

To make illustration of the algorithm simpler, we replace selection of left or right matrix multiplication with a transpose-based operation on gradients. After each local gradient $\nabla F_i(\bm{X}_t;{\bm{\xi}}_t^{(i)})$ is computed immediately, each node transposes the gradient to  $\nabla F_i(\bm{X}_t;{\bm{\xi}}_t^{(i)})^\top\in\mathbb{R}^{n\times m}$ when $m>n$ and keep the shape $\nabla F_i(\bm{X}_t;{\bm{\xi}}_t^{(i)})\in\mathbb{R}^{m\times n}$ when $m\le n$ . Then communication and compression are performed on this transposed form $n\times m$. Once the global compressed gradients are aggregated and reconstructed, they are transposed back to the original shape $m\times n$ for further update on optimizer state.  Algorithm~\ref{alg:greedylore_variant} summarizes the complete procedure.

\vspace{1mm}
\begin{algorithm}[H]
\caption{Distributed Adam-type GaLore} 
\label{alg:ce-galore}
\SetKwFunction{lazysvd}{Lazy-SVD}
\SetKwProg{sub}{Subroutine}{}{}

\vspace{1pt}
\JustifyAlgo{\hspace{-4.95mm} \textbf{Input}: $N$ nodes, learning rate $\gamma$, number of total iterations $T$, subspace changing frequency $\tau$, rank $r$, $\beta_1, \beta_2$ for Adam. Error buffer $\bm{E}_{-1} = \bm{0} \in \mathbb{R}^{m \times n}$,  weight $\bm{X}_{0} \in \mathbb{R}^{m \times n}$, state for subspace optimizer $\bm{M}_{-1} = {0}, \bm{V}_{-1} = {0} \in \mathbb{R}^{r \times n}$ and projection matrix $\bm{P}_{-1} \in \mathbb{R}^{m \times r}$ with $m \le n$. }\\
\vspace{2pt}
\textbf{Output}: Sequence of model weights $\{\bm{X}_t\}_{t=0}^{T+1}$.

\For{$t=0,\ldots,T$}{
    \textbf{(On $i$-th node)} \\
    $\bm{G}_t^{(i)} \gets \nabla F_i({\bm{X}}_t; {\bm{\xi}}_t^{(i)})$ with local data ${\bm{\xi}}_t^{(i)}$. \\
  
    \uIf{$\mathsf{mod}(t, \tau) = 0$}{
        $\bm{G}_t \gets \frac{1}{N}\sum_{i=1}^{N}{\bm{G}_t}^{(i)} $.\hfill (All-Reduce) \\
        $\bm{U},\bm{\Sigma}, \bm{V} \gets \mathrm{\bf SVD}(\bm{G}_t)$ and $\bm{P}_t \gets \bm{U}_{:,:r}$. \\
        \vspace{-0.25mm}
        $\bm{M}_t \gets \bm{0}, \bm{V}_t \gets \bm{0}$.
    }\Else{
        $\bm{P}_t \gets \bm{P}_{t-1}$.
    }
    $\bm{R}_t^{(i)} \gets {\bm{P}_t}^\top{\bm{G}}_t^{(i)}$. \\
    ${\bm{R}}_t \gets \frac{1}{N}\sum_{i=1}^{N}{\bm{R}}_t^{(i)}$. \hfill (All-Reduce)   \\
  
    $\hat{\bm{R}_t}, \bm{M}_{t}, \bm{V}_{t} \gets \mathsf{\bf Adam}\mbox{-}\mathsf{\bf Update}(\bm{R}_t, \bm{M}_{t-1}, \bm{V}_{t-1}) $.\\
    $\bm{X}_{t+1} \gets \bm{X}_{t} - \gamma \bm{P}_t\hat{\bm{R}_t}$. \\
    }
    \Return $\{\bm{X}_{t}\}_{t=0}^{T+1}$.

\vspace{2mm}
\sub{$\mathsf{\bf Adam}\mbox{-}\mathsf{\bf Update}(\bm{R}_t, \bm{M}_{t-1}, \bm{V}_{t-1})$}{
    $\bm{M}_t \gets (1-\beta_1)\bm{M}_{t-1} + \beta_1 \bm{R}_t$. \\
    $\bm{V}_t \gets (1-\beta_2)\bm{V}_{t-1} + \beta_2 \bm{R}_t\odot\bm{R}_t$. \\
    \Return $\frac{\gamma}{\sqrt{\bm{V}_t} + \epsilon} \odot \bm{M}_t, \bm{M}_t, \bm{V}_t$.
}
\end{algorithm}

\vspace{-1mm}
\section{GaLore as Communication-Efficient Optimizer}\label{sec:dist_galore}
GaLore~\cite{zhao2024galore} is a subspace-based optimizer originally proposed to reduce memory consumption during the training of deep learning models. However, under a data-parallel framework, it can be naturally extended to a communication-efficient variant. The distributed variant of Adam-type GaLore algorithm is presented in Algorithm~\ref{alg:ce-galore}.

In this distributed implementation, the low-rank projection is maintained via the Lazy-SVD subroutine, and the overall structure closely follows that of Algorithm~\ref{alg:prelimnary}. For clarity, we explicitly expand the Lazy-SVD operation within the main optimization loop. The key distinction between this variants and Algorithm~\ref{alg:prelimnary} lies in the state representation: the former operates on a compact subspace state, whereas the latter retains the full-dimensional optimizer state. Consequently, Algorithm~\ref{alg:ce-galore} can recover the standard mini-batch single-node GaLore update as a special case, which is a property not shared by Algorithm~\ref{alg:prelimnary}. Nevertheless, employing the subspace state may introduce slower convergence when training models at very low rank, as demonstrated in Figure~\ref{fig:os}.

\section{Missing Experimental Details}\label{sec:hyper-param}

In this section, we show the training details of our experiments in section~\ref{sec:exp} for reproducion.

\noindent \textbf{Pre-training tasks on CIFAR datasets}
We pre-train ResNet-18 models on CIFAR-10 and CIFAR-100 datasets for both 40 epochs on a 4 × 4090 24GB GPUs cluster. The experiments are configured with global batch size 4 × 32 in PyTorch DDP framework. We start the compression at 500 iterations for CIFAR-10 and 4 000 iterations for CIFAR-100 following the warm-up convention in PowerSGD\cite{vogels2019powersgd}. We use max learning rate of $5e-3, 5e-4$ with cosine annealing scheduler, subspace switching frequency $750, 1200$ respectively for training on CIFAR-10 and CIFAR-100 datasets. For low-rank algorithm in \ref{fig:cifar}, the compress rank is set to 64 in all settings.

\noindent \textbf{Fine-tuning tasks on GLUE datasets}
We fine-tune pre-trained RoBERTa-base models on GLUE benchmarks for 10 epochs on a 4 × 4090 24GB GPUs cluster with data parallelism at 4. We use max sequence length of 256, start-compress iterations of 1 000 and learning rate scheduler with cosine decaying to 0 and warm-up fraction of $10\%$. We also search the subspace switching frequencies in $\{200, 500\}$. The batch size and learning rate with different ranks at $r=8$ and $r=16$ for each task are shown in Table \ref{table:glue-hyper}.

\noindent \textbf{Pre-training tasks on C4 datasets}
We pre-train different size of LLaMA models on C4 benchmarks with 4 × NVIDIA A800 80GB GPUs. Training token budgets for each model size were allocated according to the Chinchilla scaling law\cite{hoffmann2022training} following \cite{zhao2024galore, robert2024ldadam}. The detailed hyper-parameters are illustrated in Table ~\ref{table:c4-hyper}.

\end{multicols}

\begin{table*}[!htbp]
\centering
\begin{tabular}{lcccccccc}
\toprule
& SST-2 & CoLA & MRPC & STS-B & RTE & QNLI & QQP & MNLI \\
\midrule
Rank & 8 & 8 & 8 & 8 & 8 & 8 & 8 & 8 \\
Learning Rate & $2e-5$ & $2.4e-5$ & $2e-5$ & $2e-5$ & $2e-5$ & $2e-5$ & $2e-5$ & $2e-5$ \\
Total Batch Size & 16 & 16 & 16 & 16 & 16 & 16 & 16 & 16 \\
Batch Size per Divice & 4 & 4 & 4 & 4 & 4 & 4 & 4 & 4 \\
\bottomrule
\toprule
 & SST-2 & CoLA & MRPC & STS-B & RTE & QNLI & QQP & MNLI \\
\midrule
Rank & 16 & 16 & 16 & 16 & 16 & 16 & 16 & 16 \\
Learning Rate & $2e-5$ & $2e-5$ & $2.4e-5$ & $2.4e-5$ & $2e-5$ & $2e-5$ & $2.4e-5$ & $2e-5$ \\
Total Batch Size & 16 & 16 & 16 & 16 & 16 & 32 & 16 & 16 \\
Batch Size per Divice & 4 & 4 & 4 & 4 & 4 & 8 & 4 & 4 \\
\bottomrule
\end{tabular}
\vspace{0.2cm}
\caption{\small Hyperparameter settings for fine-tuning RoBERTa-Base model on the GLUE benchmark.}
\label{table:glue-hyper}
\end{table*}

\begin{table*}[!htbp]
\centering
\begin{tabular}{lccccccccc}
\toprule
& \multicolumn{3}{c}{\textbf{Llama 60M}} & \multicolumn{3}{c}{\textbf{Llama 130M}} & \multicolumn{3}{c}{\textbf{Llama 350M}} \\
\cmidrule(lr){2-4} \cmidrule(lr){5-7} \cmidrule(lr){8-10}
& Adam & GreedyLore & GaLore & Adam & GreedyLore & GaLore & Adam & GreedyLore & GaLore \\
\midrule
Training Steps & \multicolumn{3}{c}{10000} & \multicolumn{3}{c}{20000} & \multicolumn{3}{c}{60000} \\
Warm-up Steps & \multicolumn{3}{c}{1000} & \multicolumn{3}{c}{2000} & \multicolumn{3}{c}{6000} \\
Maximum Length & \multicolumn{3}{c}{256} & \multicolumn{3}{c}{256} & \multicolumn{3}{c}{256} \\
Batch Size & \multicolumn{3}{c}{512} & \multicolumn{3}{c}{512} & \multicolumn{3}{c}{512} \\
Batch Size per Device & \multicolumn{3}{c}{128} & \multicolumn{3}{c}{128} & \multicolumn{3}{c}{128} \\
Total Training Tokens & \multicolumn{3}{c}{1 310 720 000} & \multicolumn{3}{c}{2 621 440 000} & \multicolumn{3}{c}{7 208 960 000} \\
\midrule
Learning Rate & \multicolumn{3}{c}{$\{1.5e-3, 2.5e-3, 5e-3\}$} & \multicolumn{3}{c}{$\{1.5e-3, 2.5e-3, 5e-3\}$} & \multicolumn{3}{c}{$\{1.5e-3, 2.5e-3, 5e-3\}$} \\
Warm-up Scheduling & \multicolumn{3}{c}{linear from 0\%} & \multicolumn{3}{c}{linear from 0\%} & \multicolumn{3}{c}{linear from 0\%} \\
Learning Rate Scheduling & \multicolumn{3}{c}{cosine to 10\%} & \multicolumn{3}{c}{cosine to 10\%} & \multicolumn{3}{c}{cosine to 10\%} \\
Weight Decay & 0.0 & 0.0 & 0.0 & 0.0 & 0.0 & 0.0 & 0.0 & 0.0 & 0.0 \\
Gradient Clipping & 1.0 & 1.0 & 1.0 & 1.0 & 1.0 & 1.0 & 1.0 & 1.0 & 1.0 \\
\midrule
Error Feedback & \xmark & \cmark & \xmark & \xmark & \cmark & \xmark & \xmark & \cmark & \xmark \\
Subspace Frequency & - & 200 & 200 & - & 200 & 200 & - & 200 & 200 \\
\bottomrule
\end{tabular}
\vspace{0.2cm}
\caption{\small Hyperparameter settings for pre-training LLaMA model on the C4 dataset.}
\label{table:c4-hyper}
\end{table*}

\end{document}